\documentclass[10.5pt]{article}
\usepackage{pdfpages}
\usepackage{tikz}
\usepackage[text={6.5in,9in}]{geometry} 
\usepackage{amsfonts, amssymb, amsthm, amsmath, mathrsfs, mathtools}
\usepackage{natbib, enumitem, setspace, multicol, multirow, booktabs}
\usepackage{graphicx, float, caption, subcaption}
\usepackage{nccmath, mathabx, relsize}
\usepackage{graphicx}
\usepackage{subcaption} 

\usepackage{xcolor}
\usepackage{tikz}
\usetikzlibrary{arrows.meta,positioning}
\usepackage{color}
\definecolor{dbl}{rgb}{0.1,0.0,0.97}
\usepackage{array}
\usepackage{parskip}
\usepackage{siunitx}
\newcommand{\setE}{\mathbb{E}}

\newcommand{\setR}{\mathbb{R}}
\newcommand{\setS}{\mathbb{S}}

\newcommand{\vu}{\mathbf{u}}

\newcommand{\vx}{\mathbf{x}}

\newcommand{\mX}{\mathbf{X}}

\newcommand{\ones}{\boldsymbol{1}}
\newcommand{\zeros}{\boldsymbol{0}}
\newcommand{\vell}{\boldsymbol{\ell}}

\newcommand{\trace}{\mathrm{tr}}

\DeclarePairedDelimiter\norm{\lVert}{\rVert}

\theoremstyle{plain}
\newtheorem{theorem}{Theorem}[section]
\newtheorem*{theorem*}{Theorem}
\newtheorem{lemma}[theorem]{Lemma}

\theoremstyle{definition}
\newtheorem{definition}{Definition}
\theoremstyle{definition}

\theoremstyle{remark}

\usepackage{algorithmic}
\usepackage{algorithm}
\usepackage[pdftex,hypertexnames=false,linktocpage=true,colorlinks]{hyperref}
\hypersetup{colorlinks=true,linkcolor=blue,anchorcolor=blue,citecolor=blue,filecolor=blue,urlcolor=blue,bookmarksnumbered=true}

\numberwithin{equation}{section}
\numberwithin{theorem}{section}
\newcommand{\bib}{\bibliography{ref}\bibliographystyle{apalike}}

\title{Co-Hub Node Based Multiview Graph Learning with Theoretical
Guarantees }
\author{%
Bisakh Banerjee\thanks{B.B. and T.M. are with the Department of Statistics and Probability,
Michigan State University, East Lansing, MI 48824 USA.(e-mail: banerj40@msu.edu;maiti@msu.edu)}%
\and
Mohammad Alwardat\thanks{M.A. and S.A are with the Department of Electrical and Computer
Engineering, Michigan State University, East Lansing, MI 48824 USA.(e-mail:alwardat@msu.edu; ; aviyente@egr.msu.edu). This work was supported in part by the National Science Foundation under Grants
CCF-2211645 and ECCS-2430516.}%
\and
Tapabrata Maiti\footnotemark[1]%
\and
Selin Aviyente\footnotemark[2]%
}%

\begin{document}

\maketitle

\begin{abstract}
Identifying the graphical structure underlying the observed multivariate data is essential in numerous applications. Current methodologies are predominantly confined to deducing a singular graph under the presumption that the observed data are uniform. However, many contexts involve heterogeneous datasets that feature multiple closely related graphs, typically referred to as multiview graphs. Previous research on multiview graph learning promotes edge-based similarity across layers using pairwise or consensus-based regularizers. However, multiview graphs frequently exhibit a shared node-based architecture across different views, such as common hub nodes. Such commonalities can enhance the precision of learning and provide interpretive insight. In this paper, we propose a co-hub node model, positing that different views share a common group of hub nodes. The associated optimization framework is developed by enforcing structured sparsity on the connections of these co-hub nodes. Moreover, we present a theoretical examination of layer identifiability and determine bounds on estimation error. The proposed methodology is validated using both synthetic graph data and fMRI time series data from multiple subjects to discern several closely related graphs.

\end{abstract}


\section{Introduction}
\label{sec:intro}
Many real-world data are represented through relationships between data samples or features, i.e., a graph
structure \cite{newman2018networks}. 
 While a predefined graph structure accompanies specific datasets, such as social networks, there exist a multitude of applications where such a structure is not immediately accessible. For example, in brain networks \cite{gao2021smooth}, the graph structure underlying the observed time series data is not easily discernible. In these instances, it becomes imperative to infer the graph topology to facilitate the analysis of the data and model the relationships.

Current graph inference methodologies are predominantly constrained to homogeneous datasets, where observed graph signals are presumed to be identically distributed and confined to a single graph paradigm. In various applications, data may exhibit heterogeneity and originate from multiple graphs, referred to as multiview graphs. In such contexts, jointly learning the topology of views by integrating inter-view relationships has been shown to improve performance \cite{tsai2022joint, danaher2014joint, navarro2022joint,karaaslanli2025multiview}. Conventional joint graphical structure inference techniques are primarily based on Gaussian Graphical Models. These methodologies extend the graphical lasso \cite{friedman2008sparse} to a joint learning framework, in which they derive the precision matrices of multiple related Gaussian graphical models (GGMs) by incorporating various penalties to leverage the common features shared across views \cite{guo2011joint, danaher2014joint, lee2015joint, mohan2014node, ma2016joint, huang2015joint}. These approaches are limited by their assumption that the observed graph signals follow a Gaussian distribution and by their lack of enforcement of graph structure constraints when learning precision matrices. Recently, these joint learning strategies have been adapted to infer multiple graph Laplacian matrices \cite{yuan2023joint,zhang2024graph,karaaslanli2025multiview}. These frameworks enforce edge-based similarity across the views through either pairwise similarity or similarity to a consensus graph. In numerous settings, such as brain networks, the similarity between views can be more aptly elucidated through the common structure of a limited number of nodes, e.g., hub nodes. This method of modeling similarities imposes a structure and offers an intuitive interpretation of joint graph learning.

In this paper, we expand on our initial work in \cite{alwardat2025co} by introducing significant advances that develop a GSP-based method for learning multiview graph Laplacians with shared hub nodes. The method learns Laplacians by enforcing graph-signal smoothness and adding a node-based regularizer that captures cross-view similarity. In this paper, we first introduce a multi-block ADMM framework for solving our optimization problem along with a proof of convergence. Following this, we establish upper bounds on the estimation error of the multiview Laplacian matrices and demonstrate the identifiability of the individual views. Lastly, we present a principled hyperparameter selection procedure using the Bayesian Information Criterion and present the performance of our algorithm across different random graph models and signals, along with findings from multi-subject fMRI data.

The main contributions of this work are:
\begin{itemize}
\item{A new multiview graph learning framework leveraging signal smoothness and node-based similarity. We introduce the first approach that models cross-view similarity through a co-hub structure, rather than edge-based similarity as in prior multiview Laplacian learning methods \cite{karaaslanli2025multiview,yuan2023joint,zhang2024graph,zhang2023graph}. As a result, our method enables the learned graphs to share meaningful structural components across views.}
\item{A flexible alternative to GGM-based joint graph inference. Unlike classical statistical approaches that rely on Gaussian graphical models, our framework does not require Gaussian assumptions on the data, offering a broader and more versatile tool for multiview co-hub graph learning.}
\item{Theoretical guarantees on identifiability and estimation error. We establish conditions for layer identifiability and derive an estimation error bound that highlights the dependence of the error on both sample size and the structural properties of the underlying graphs.}
\end{itemize}
\section{Related Work}
There is a substantial body of literature
 for joint estimation of multiple graphical structures from high-dimensional data, motivated by applications in biomedicine and social science \cite{guo2011joint,danaher2014joint,friedman2008sparse,guo2015estimating,lee2015joint,tan2014learning,tarzanagh2018estimation}. Since in most applications the number of model parameters
to be estimated far exceeds the available sample size, the assumption of sparsity is made and imposed through regularization of the learned graph, e.g.,   $\ell_{1}$
penalty on the edge weights \cite{friedman2008sparse,meinshausen2006high,peng2009partial}. This approach encourages sparse uniform graph structures, which may not be  suitable for real-world applications that are not uniformly sparse. 

Many real-world networks have intrinsic  structure \cite{guo2011joint,danaher2014joint,tarzanagh2018estimation}. One commonly encountered one is the existence of
a densely connected subgraph or community.
An important part of the literature deals
with the estimation of hidden communities of nodes while inferring the graph structure. 
Initial work focused on either inferring connectivity information \cite{marlin2009sparse} or performing graph estimation in case the
connectivity or community information is known {\em a priori} \cite{danaher2014joint,guo2011joint,gan2019bayesian,ma2016joint,lee2015joint},
but not both tasks simultaneously. Recently,   the two tasks have been jointly addressed, estimating the graph and detecting the community structure \cite{kumar2020unified,hosseini2016learning,hao2018simultaneous,tarzanagh2018estimation,gheche2020multilayer,pircalabelu2020community}. 

More recently, other structures across multiview graphs have been taken into account \cite{mohan2014node}. One example is joint learning of multiple graphical structures with common hub nodes. Hub nodes play an important role in biological networks such as  gene regulatory \cite{jeong2001lethality} and functional brain networks \cite{ortiz2025learning}. 
The identification of hub nodes in these networks is of particular interest, as they represent attractive targets for treatment
due to their critical position within the network structure. The problem of learning graphical structures with a small number of hub nodes was first addressed in the context of power-law graphs \cite{liu2011learning,defazio2012convex,tandon2014learning}. These methods focus on learning a single graphical model with assumptions on the observed signals and the underlying graph, e.g., the Gaussian graphical model, the Gaussian covariance graph model and the Ising model \cite{tan2014learning}. Joint graph learning with hub nodes was studied using GGM \cite{mohan2014node,mcgillivray2020estimating,huang2024njgcg} and Bayesian models \cite{kim2019bayesian}.
However, all of this prior work relies on the GGM assumption.

More recently, joint graph learning problem has been addressed in the context of graph signal processing \cite{navarro2022joint,navarro2022jointb,karaaslanli2025multiview,yuan2023joint,zhang2023graph,zhang2024graph}. In this line of work, the problem of multiview
graph learning is formulated with assumptions about the graph signals with respect to the underlying graphical structure. In \cite{navarro2022joint}, the signals are assumed to
be stationary and pairwise similarity between the views is used
to regularize the objective function. In \cite{karaaslanli2025multiview,zhang2024graph,zhang2023graph}, the authors propose a multiview graph learning method based on smoothness assumption where the similarity between the views is ensured either through pairwise \cite{zhang2024graph} or consensus \cite{karaaslanli2025multiview} based regularization. However, these approaches only quantify edge-wise similarity, thus not capturing the node-based structures that are intrinsic to the graphs. 

\section{Background}
\label{sec:background}
\subsection{Notations}

In this paper, we denote a vector with bold lower case notation, $\mathbf x$, and a matrix with bold upper case notation,  $\mathbf{A} \in \mathbb{R}^{m \times n}$, with the $(i,j)$ th entry of the matrix $\mathbf{A}$ denoted as $A_{ij}$ or $[A]_{ij}$. We  define $\mathbf{A}_{\cdot j}$ as the $j$-th column of the matrix $\mathbf{A}$ for $j\in\{1,2,\ldots,n\}$. For $\mathbf{A}\in\mathbb{R}^{m\times n},\quad
\operatorname{supp}(\mathbf{A})\coloneqq\{(i,j)\in[m]\times[n]\;:\;A_{ij}\neq 0\}$,
\quad where $[m]:=\{1,\dots,m\},\ [n]:=\{1,\dots,n\}.$ All-one vector, all-zero vector and identity matrix are shown as $\mathbf{1}, \mathbf{0}$, and $\mathbf{I}$. The trace of a square matrix $\mathbf{A} \in \mathbb{R}^{n \times n}$ is denoted as $\operatorname{tr}(\mathbf{A})= \sum_{i} A_{ii}$. The symbol $\odot$ refers to the Hadamard product (element-wise) of two matrices, i. e., for $\mathbf{A} \in \mathbb{R}^{m\times n}$ and $\mathbf{B} \in \mathbb{R}^{m\times n}$, $[\mathbf{A}\circ \mathbf{B}]_{ij}=A_{ij}B_{ij}$.
$\mathbf{A}^{\dagger}$ denotes the Moore-Penrose pseudoinverse of the matrix $\mathbf{A}\in\mathbb{R}^{m\times n}$. 
For any vector $\mathbf{x}$ and any matrix $\mathbf{A}$, $\left\|\mathbf{x}\right\|_{2}$ and $\left\|\mathbf{A}\right\|_{2}$ denote the $\ell_2$ norm and the spectral norm respectively. The Frobenius norm of a matrix is defined as $\left\| \mathbf{A} \right\|_F = \sqrt{\sum_{i,j} A_{ij}^2}$. $\ell_{2,1}$ norm of a matrix is defined as \(\left\|\mathbf{A}\right\|_{2,1} = \sum_{j} \sqrt{\sum_{i} A_{ij}^2}\).    For matrices $\mathbf{A}_k \in \mathbb{R}^{m_k \times n_k}, k\in\{1, \ldots, K\}$, the block-diagonal matrix $\operatorname{bldiag}\left(\mathbf{A}_1, \ldots, \mathbf{A}_K\right)\in \mathbb{R}^{\left(\sum_{k=1}^K m_k\right) \times\left(\sum_{k=1}^K n_k\right)}$ is defined as 
$$
\operatorname{bldiag}\left(\mathbf{A}_1, \ldots, \mathbf{A}_K\right):=\left[\begin{array}{cccc}
\mathbf{A}_1 & & & \mathbf{0} \\
& \mathbf{A}_2 & & \\
& & \ddots & \\
\mathbf{0} & & & \mathbf{A}_K
\end{array}\right].
$$ 
For a function of $n\in \mathbb{N}$, $T(n)=\mathcal{O}(n)$ implies that there exists $C>0, n_0 \in \mathbb{N}$ such that $T(n) \leq C n$ for all $n \geq n_0$.
\subsection{Graphs and Single View Graph Learning}
\label{ssec:single-graph-learning}
An undirected weighted graph is represented as $G=(V, E, \mathbf{A})$ where $V$ is the node set with cardinality $|V| = n$ and $E$ is the edge set. $\mathbf{A} \in \setR^{n\times n}$ is the adjacency matrix, where $A_{ij} = A_{ji}$ is the edge weight. $\mathbf{d} = \mathbf{A} \ones$ is the degree vector and $\mathbf{D} = \text{diag}(\mathbf{d})$ is the diagonal degree matrix. 
The Laplacian matrix is defined as $\mathbf{L} = \mathbf{D} - \mathbf{A}$ with its eigendecomposition as $\mathbf{L} = \mathbf{U} \mathbf{\Lambda} \mathbf{U}^{\top}$, where the columns of $\mathbf{U}$ are the eigenvectors, and $\mathbf{\Lambda}$ is the diagonal matrix of eigenvalues with $0 = \Lambda_{11} \leq \Lambda_{22} \leq \dots \leq \Lambda_{nn}$.

A graph signal defined on $G$ is a function $x : V \to \mathbb{R}$ and can be represented as a vector $\mathbf x \in \mathbb{R}^{n}$ where $x_{i}$ is the signal value on node $i$. The eigenvectors and eigenvalues of $\mathbf{L}$ can be used to define the graph Fourier transform (GFT), i.e., $\widehat{\mathbf x} = \mathbf{U}^\top \mathbf x$ where $\hat{x}_{i}$ is the Fourier coefficient in the $i$th frequency component $\Lambda_{ii}$. $\mathbf x$ is a smooth graph signal if most of the energy of $\widehat{\mathbf x}$ lies in low-frequency components, which can be quantified using the total variation defined in terms of spectral density as:
\begin{equation}
    \label{eq:smoothness}
\operatorname{tr}\left(\widehat{\mathbf x}^\top \mathbf{\Lambda} \widehat{\mathbf x}\right) = \operatorname{tr}\left(\mathbf{x}^\top \mathbf{U} \mathbf{\Lambda} \mathbf{U}^T \mathbf x\right) = \operatorname{tr}\left(\mathbf x^\top \mathbf{L} \mathbf x\right).
\end{equation}

An unknown graph $G$ can be learned from a set of graph signals based on some assumptions about the relation between the observed graph signals and the underlying graph structure. 
Dong et. al. \cite{dong2016learning} proposed to learn $G$ assuming that the graph signals are smooth with respect to $G$, which can be quantified using \eqref{eq:smoothness}. Given the observed graph signals, $\mathbf{X} \in \mathbb{R}^{n \times p}$, the Laplacian matrix, $\mathbf{L}$, can be learned as:
\begin{equation}
\begin{split}
\label{eq:single_graph_learning}
    \min_{\mathbf{L}} &\ \operatorname{tr}\left({\mathbf{X}}^\top \mathbf{L} \mathbf{X}\right) + \alpha \left\|\mathbf{L}\right\|_F^2 \hspace{0.5em}
    \textrm{s.t.}\hspace{0.5em}   \mathbf{L} \in \mathbb{L} \textrm{ and } \operatorname{tr}(\mathbf{L}) = 2n,
\end{split}
\end{equation}
where the first term quantifies the total variation of graph signals and the second term controls the density of the learned graph. $\mathbf{L}$ is constrained to be in $\mathbb{L}=\{\mathbf{L} : \mathbf{L} \succeq \mathbf{0}, L_{ij} = L_{ji} \leq 0\ \forall i\neq j,\ \mathbf{L} \mathbf{1} = \mathbf{0}\}$, which is the set of valid Laplacians. The second constraint is added to prevent the trivial solution.


\section{Co-Hub Node Based Multiview Graph Learning (CH-MVGL)}
\subsection{Problem Formulation}
Given a set of signal samples for each view, $\mathbf{X}^{k}= \left[ \mathbf{X}^{k}_{\cdot 1}, \ldots ,\mathbf{X}^{k}_{\cdot d_{k}}\right] $ where $\mathbf{X}^{k}_{\cdot i} \in \mathbb{R}^{n}$  with $n$ and $d_{k}$ corresponding to the number of nodes and signal samples in view $k$, respectively, the goal is to learn the individual graph structures, i.e., the graph Laplacians, $\mathbf {L}^{k}$. Assuming that the multiview graphs share  $h \ll n$ co-hub nodes, the individual graph Laplacians can be decomposed as $\mathbf {L}^{k}=\mathbf{S}^{k}+\mathbf{V}+\mathbf{V}^{\top}$, where $\mathbf{S}^{k}$ is the unique and sparse part of each view and $\mathbf{V}+\mathbf{V}^{\top}$ is the common connectivity pattern of hub nodes across views. The problem of learning the individual graph Laplacians with the smoothness assumption can then be formulated as
\begin{equation}
{\begin{split}
    &\min_{\substack{\mathbf L^{k},\mathbf{V},\mathbf S^{k}}} \sum^K_{k=1} \left\{\operatorname{tr}\left( {\mathbf{X}^{k}}^\top\mathbf{L}^{k} {\mathbf{X}^{k}}\right) + \gamma_1 \left\| \mathbf{L}^{k}-\mathbf{L}^{k}\odot\mathbf{I}\right\|_F^2  - \gamma_2 \operatorname{tr}\left(\log\left(\mathbf{I}\odot\mathbf{L}^{k}\right)\right) +\gamma_4\left\|\mathbf{S}^{k}\right\|_{F}^2 \right\} +\gamma_3\left\| \mathbf{V}\right\|_{2,1} \\
    &\,\,\,\,\,\ \text{s.t.}  \hspace{0.45em}\,\,\,\,\,\ \mathbf{L}^{k}\succeq \mathbf{0}, \mathbf L^{k}\cdot \mathbf 1 =\mathbf{0}, \hspace{1em}\mathbf{S}^{k}\succeq \mathbf{0}, \mathbf S^{k}\cdot \mathbf 1 =\mathbf{0},\hspace{1em}\mathbf{L}^{k}-\mathbf{S}^{k}=\mathbf{V}+{\mathbf{V}}^\top, \quad k=1,2,\ldots, K;
\end{split}}
\label{eq:objFunc}
\end{equation} 

\noindent where the first term quantifies the total variation of the observed signals, $\mathbf{X}^{k}$, with respect to the underlying graph  Laplacian $\mathbf{L}^{k}$, the second term is equivalent to the Frobenius norm of the off-diagonal elements and controls the sparsity of the learned graphs, the third term applies a logarithmic penalty to the degree of the learned graphs, $(\mathbf{I}\odot\mathbf{L}^{k})$, to ensure connectivity \cite{kalofolias2017learning}, the fourth term controls the sparsity of unique part of each view and the last term is used to learn the co-hub nodes using the $\ell_{2,1}$ norm.
$\ell_{2,1}$-norm of $\mathbf{V}$ ensures that the number of co-hub nodes is small, i.e.,   the number of columns of $\mathbf{V}$ with non-zero $\ell_{2}$-norm is minimized.

Following \cite{zhao2019optimization}, the set of constraints in Eq. \eqref{eq:objFunc} can be written equivalently in the following form:
\begin{equation}
    \begin{split}
        &\mathbf{L}^{k}\succeq \mathbf{0}, \mathbf L^{k}\cdot \mathbf 1 =\mathbf{0}\iff\quad\mathbf P \mathbf{E}^{k} \mathbf P^\top, \mathbf{E}^{k}\succeq \mathbf{0}, \\& \mathbf{S}^{k}\succeq \mathbf{0}, \mathbf S^{k}\cdot \mathbf 1 =\mathbf{0}\iff\quad\mathbf P \mathbf{\Xi}^{k} \mathbf P^\top, \mathbf{\Xi}^{k}\succeq \mathbf{0},
    \end{split}
\end{equation}
\noindent where $\mathbf P\in\mathbb{R}^{n\times (n-1)}$ is the orthogonal complement of the vector $\mathbf 1$, i.e., $\mathbf P^\top \mathbf P =\mathbf I$ and $\mathbf P ^\top \mathbf 1 =\mathbf 0$, and $\mathbf{E}^{k}$s and $\mathbf{\Xi}^{k}$s are positive semi-definite matrices of dimension $ {(n-1)\times (n-1)}$. Note that the choice of $\mathbf{P}$ is nonunique. The proposed objective function can be solved using the multi-block Alternating Direction Method of Multipliers (ADMM). 
\subsection{Optimization Problem}
To solve the optimization problem in Eq. \eqref{eq:objFunc}, we introduce the auxiliary variables, $\mathbf{W}$, $\mathbf{C}^k$, $\mathbf{G}$, $\mathbf{\Gamma}^k$, $\mathbf{\Psi}^k$, and $\mathbf{Z}^k$, to decouple the different variables and rewrite \eqref{eq:objFunc}  as follows: 
\begin{equation}
{\begin{split}
    &\min_{\substack{\mathbf E^{k},\mathbf{Z}^k,\mathbf \Xi^{k},\mathbf{\Psi}^k, \mathbf{G},\\ \mathbf{V},\mathbf{W},\mathbf{C}^k,\mathbf{\Gamma}^k}} \sum^K_{k=1} \left\{\operatorname{tr}\left( \mathbf{B}^k\mathbf{E}^k\right) + \gamma_1 \left\| \mathbf{\Psi}^k\right\|_F^2 - \gamma_2 \operatorname{tr}\left(\log\left(\mathbf{Z}^k\right)\right) +\gamma_4\left\|\mathbf P \mathbf{\Xi}^k \mathbf P^\top\right\|_F^2\right\}+\gamma_3\left\| \mathbf{G}\right\|_{2,1} \\
    & \text{s.t.}  \hspace{1em} \mathbf{C}^k=\mathbf P \mathbf{E}^k \mathbf P^\top, \mathbf{Z}^k=\mathbf{C}^k\odot\mathbf{I}, \mathbf{\Gamma}^k=\mathbf P \mathbf{\Xi}^k \mathbf P^\top, \mathbf{G}={\mathbf{V}}, \mathbf{C}^k-\mathbf{\Gamma}^k=\mathbf{V}+{\mathbf{W}},  {\mathbf{W}^\top}=\mathbf{V}, \mathbf{\Psi}^k=\mathbf P \mathbf{E}^k \mathbf P^\top-\mathbf{Z}^k,
\end{split}}
\label{eq:objFunc11}
\end{equation} 
\noindent where ${\mathbf{B}^k}=\mathbf{P}^\top{\mathbf{X}^k}{\mathbf{X}^k}^\top\mathbf{P}$. In the proposed multi-block ADMM framework, the variables are partitioned as 
$ \{\mathbf{E}^k,\, \mathbf{\Xi}^k, \, \mathbf{V}\}$, 
$ \{\mathbf{Z}^k,\, \mathbf{\Gamma}^k\}$, 
$ \{\mathbf{C}^k,\, \,  \, \, \mathbf{\Psi}^k\}$, 
and 
$ \{\mathbf{W}, \, \, \mathbf{G}\}$ \cite{parikh2014proximal}. Each block can be updated independently within the ADMM framework.  The augmented Lagrangian corresponding to \eqref{eq:objFunc11} can be written as follows: 
\begin{align}
    \min_{\substack{\mathbf E^{k},\mathbf{Z}^k,\mathbf \Xi^{k},\mathbf{\Psi}^k, \mathbf{G},\notag\\ \mathbf{V},\mathbf{W},\mathbf{C}^k,\mathbf{\Gamma}^k}} &\sum^K_{k=1}\left\{ \operatorname{tr}\left(\mathbf{B}^k \mathbf{E}^k\right) + \gamma_1 \left\| \mathbf{\Psi}^k\right\|_F^2 - \gamma_2 \operatorname{tr}\left(\log\left(\mathbf{Z}^k\right)\right)+\gamma_4\left\|\mathbf P \mathbf{\Xi}^k \mathbf P^\top\right\|_F^2 +\frac{\alpha}{2}\left\Vert\mathbf{C}^k-\mathbf P \mathbf{E}^k \mathbf P^\top+\frac{\mathbf{Y}^k}{\alpha}\right\Vert_F^2\right\}\notag \\
    & +\sum_{k=1}^{K}\frac{\alpha}{2}\left\{\left\Vert\mathbf{Z}^k-\mathbf{I}\odot\mathbf{C}^k+\frac{\mathbf{J}^k}{\alpha}\right\Vert_F^2  +\left\Vert\mathbf{\Gamma}^k-\mathbf P \mathbf{\Xi}^k \mathbf P^\top+\frac{\mathbf{T}^k}{\alpha}\right\Vert_F^2+\left\Vert\mathbf{C}^k-\mathbf{\Gamma}^k-\mathbf{V}-\mathbf{W}+\frac{\mathbf{M}^k}{\alpha}\right\Vert_F^2\right\}\notag\\
    &+\sum_{k=1}^{K}\left[\frac{\alpha}{2}\left\|\mathbf{\Psi}^k-\mathbf{P}\mathbf{E}^k{\mathbf{P}}^\top+\mathbf{Z}^k+\frac{\mathbf{R}^k}{\alpha}\right\|_F^2\right]+\gamma_3\left \|\mathbf{G}\right\|_{2,1} +\frac{\alpha}{2}\left\Vert\mathbf{G}-\mathbf{V}+\frac{\mathbf{Q}}{\alpha}\right\Vert_F^2 +\frac{\alpha}{2}\left\Vert\mathbf{V}-\mathbf{W}^\top+\frac{\mathbf{N}}{\alpha}\right\Vert_F^2  
\label{eq:LagFunc}
\end{align} 

\noindent where $\mathbf{Y}^k, \mathbf{J}^k, \mathbf{T}^k,\mathbf{Q}, \mathbf{M}^{k},  \mathbf{N}$ and $\mathbf{R}^k$ are the Lagrange multipliers and $\alpha$ is the penalty parameter. \noindent This optimization problem can be solved using multi-block ADMM with the update steps for each variable as given in Appendix \ref{sec:Appendix_derivations}. The pseudocode for the proposed  CH-MVGL model is summarized in Algorithm \ref{alg: algorithm 1}. \footnote{The code for CH-MVGL is
available online: \url{https://github.com/wardat99/CH-MVGL}.}
\begin{algorithm}
\caption{CH-MVGL Model}
    \renewcommand{\algorithmicprint}{\textbf{goto}}
{\normalsize  
\begin{algorithmic}[1]
\REQUIRE
{$\mathbf{X}^k$, $\mathbf{P}$, $\gamma_1$, $\gamma_2$, $\gamma_3$, $\gamma_4$, $\alpha$.}
 \ENSURE
 Learned Laplacian matrices, $\mathbf{L}^k$.
\WHILE {not converge}
\STATE \hspace{2em}\COMMENT{First block}
\STATE Update $\mathbf{E}^k_{l+1}$ via Eq. \eqref{eq: Ek sol}. 
\STATE Update $\mathbf{\Xi}^k_{l+1}$ via Eq. \eqref{eq: Xi sol}.
\STATE Update $\mathbf{V}_{l+1}$ via Eq. \eqref{eq: V sol}.
\STATE \hspace{2em}\COMMENT{Second block}
\STATE Update $\mathbf{Z}^k_{l+1}$ via Eq. \eqref{eq: Zk sol}. 
\STATE Update $\mathbf{\Gamma}^k_{l+1}$ via Eq. \eqref{eq: Gamma sol}.
\STATE \hspace{2em}\COMMENT{Third block}
\STATE Update $\mathbf{C}^k_{l+1}$ via Eq. \eqref{eq: Ck sol}. 
\STATE Update $\mathbf{\Psi}^k_{l+1}$ via Eq. \eqref{eq: Psi sol}. 
\STATE \hspace{2em}\COMMENT{Fourth block}
\STATE Update $\mathbf{G}_{l+1}$ via Eq. \eqref{eq: G sol}.
\STATE Update $\mathbf{W}_{l+1}$ via Eq. \eqref{eq: W sol}. 
\STATE Update the Lagrangian multiplier and penality parameter via Eq. \eqref{eq: Lag multiplayer and penalty paremeter}.
\ENDWHILE
\end{algorithmic}}
\label{alg: algorithm 1}
\end{algorithm}
\subsection{Computational Complexity Analysis}
We analyze the computational complexity of Algorithm \ref{alg: algorithm 1} per iteration. Before solving the ADMM steps, for each view we construct $\mathbf{B}^k=\mathbf{P}^{\top}\left(\mathbf{X}^k \mathbf{X}^{k \top}\right) \mathbf{P}$ for a computational complexity of 
$
\sum_{k=1}^K\left[\mathcal{O}\left(n^2 d_k\right)+ \mathcal{O}\left(n^3\right)\right].  $  
 Since in general $d_k \gg n$, the pre-computation complexity will be $
\mathcal{O}\left(K n^2 d\right)$ where $d= \operatorname{max}\left\{d_1,d_2,\ldots,d_K\right\}$. For the ADMM steps, we form the matrices $
\mathbf{P}^{\top} \mathbf{Z}^k \mathbf{P}, \mathbf{P} \mathbf{E}^k \mathbf{P}^{\top},\mathbf{P}^{\top} \mathbf{\Gamma}^k \mathbf{P}$ to obtain the updates $\mathbf{E}^k, \mathbf{Z}^k$ and $\mathbf{\Xi}^k$ for each $k$. In each case, the computational complexity is $\mathcal{O}(n^3)$ per iteration. The total cost over $K$ views becomes $\mathcal{O}(Kn^3)$. For all other updates, we perform element-wise operations on the $n \times n$ matrices or take  linear combinations of them, each of which costs $\mathcal{O}(Kn^2)$ or $\mathcal{O}(n^2)$ based on whether or not the update is view-specific as given in Table \ref{tab:complexity}. 
\begin{table}[h]
\centering
\caption{Computational Complexity of ADMM Steps}
\begin{tabular}{l|l}
\hline
\textbf{Steps}                        & \textbf{Total Time Complexity}                         \\
\hline

$\mathbf{E}^k$--updates $\forall\ k$             & $\mathcal{O}(K\,n^3)$                                             \\
$\mathbf{Z}^k$--updates $\forall\ k$             & $\mathcal{O}(K\,n^3)$                                             \\
$\mathbf{\Xi}^k$--updates $\forall\ k$           & $\mathcal{O}(K\,n^3)$                                             \\
$\mathbf{C}^k,\Gamma^k$--updates $\forall\ k$    & $\mathcal{O}(K\,n^2)$                                             \\
$\mathbf{W}$--update                          & $\mathcal{O}(K\,n^2)$                                             \\
$\mathbf{V}$--update                          & $\mathcal{O}(K\,n^2)$                                             \\
$\mathbf{G}$--update                          & $\mathcal{O}(n^2)$                                                \\
Dual multiplier updates              & $\mathcal{O}(K\,n^2)$                                             \\
\hline
Total per iteration        & $\mathcal{O}(K\,n^3)$                                    \\
\hline
\end{tabular}
\label{tab:complexity}
\end{table}
\subsection{Hyperparameter Selection via BIC}
To perform hyperparameter selection in CH-MVGL, we employ the Bayesian Information Criterion (BIC), which offers a principled trade-off between data fidelity and model complexity. Specifically, we select the hyperparameter tuple $\boldsymbol{\gamma}=\left(\gamma_1, \gamma_2, \gamma_3, \gamma_4\right)$ that balances smooth signal representation with sparse and interpretable graph structures across views.

Let $\widehat{\mathbf{L}}_{\boldsymbol{\gamma}}^{k} \in \mathbb{R}^{n \times n}$ denote the Laplacian matrix learned for the $k$-th view using a fixed hyperparameter tuple $\boldsymbol{\gamma}$, and define the empirical covariance matrix as $\frac{1}{d_k} \mathbf{X}^{k} \mathbf{X}^{k \top}$. Assume that $\mathbf{X}_i^{k}$ is a sample from a multivariate normal distribution with zero mean and covariance $\left(\widehat{\mathbf{L}}_{\boldsymbol{\gamma}}^{k}\right)^{\dagger}$. This formulation captures the smoothness of the signal over the learned graph structures \cite{dong2016learning}. Under this model, the negative log-likelihood of the data in the $k$-th view (up to an additive constant) is given by:
$$
\ell_k=\frac{1}{2}\left[d_k \log \operatorname{det}^{+}\left(\widehat{\mathbf{L}}_{\boldsymbol{\gamma}}^{k}\right)+\operatorname{tr}\left(\mathbf{X}^{k \top} \widehat{\mathbf{L}}_{\boldsymbol{\gamma}}^{k} \mathbf{X}^{k}\right)\right],
$$
where $\log \operatorname{det}^{+}(\cdot)$ denotes the log-pseudo-determinant, defined as the logarithm of the product of the nonzero eigenvalues of the matrix argument. The total log-likelihood across all views is then given as $\ell(\boldsymbol{\gamma})=\sum_{k=1}^K \ell_k$.

To penalize the complexity of the model, we define the
effective degrees of freedom (df) based on the number of free
parameters in the learned matrices. As $\widehat{\mathbf{L}}_{\boldsymbol{\gamma}}^{k}$ are symmetric, the total number of parameters to estimate is equal to $n(n+1)/2$. However, the additional constraints $\widehat{\mathbf{L}}_{\boldsymbol{\gamma}}^{k}\mathbf{1}=\mathbf{0}$,  reduce the number of unknown parameters to $n(n-1)/2$. For $K$ views, there will be $Kn(n-1)/2$ number of free parameters. Since each $(i,j)$-th element of $\widehat{\mathbf{S}}_{\boldsymbol{\gamma}}^{k}$ for $k= 1,2,\ldots,K$ can be expressed as 
$
\left[\widehat{S}_{\boldsymbol{\gamma}}^{k}\right]_{ij} = \left[\widehat{L}_{\boldsymbol{\gamma}}^{k}\right]_{ij} - \left[\widehat{V}_{\boldsymbol{\gamma}}\right]_{ij} - \left[\widehat{V}_{\boldsymbol{\gamma}}\right]_{ji},
$
$df\left(\widehat{\mathbf{L}}_{\boldsymbol{\gamma}}^{k}\right)+df\left(\widehat{\mathbf{V}}_{\boldsymbol{\gamma}}\right) = df\left(\widehat{\mathbf{S}}_{\boldsymbol{\gamma}}^{k}\right)$. Thus, we only need to calculate the df corresponding to $\mathbf{V}$. In \cite{zou2007degrees}, it is shown that the unbiased estimator of the df of the penalizing parameter will be the number of elements in the activation set that is equivalent to the number of non-zero elements of the penalized matrix.
If co-hub nodes are inferred through the shared structure $\widehat{\mathbf{V}}_{\boldsymbol{\gamma}}$, additional model complexity can be captured via
$
\operatorname{df}^{*}=\#\left\{j \in\{1, \ldots, n\}:\left\|\widehat{\mathbf{V}}_{\cdot j}\right\|_2 \neq 0\right\}
$,
which is the number of  hub nodes shared across views.
The overall BIC score for a given hyperparameter tuple $\boldsymbol{\gamma}=(\gamma_1,\gamma_2,\gamma_3,\gamma_4)$ is computed as
$
\mathrm{BIC}(\boldsymbol{\gamma})=2 \ell(\boldsymbol{\gamma})+ \log (N) \cdot \mathrm{df},
$
where $N=\sum_{k=1}^K d_k \cdot n$ is the total number of observations and the total df is $Kn(n-1)/2+\operatorname{df}^{*}$. In practice, model selection is performed by evaluating $\operatorname{BIC}(\boldsymbol{\gamma})$ across a pre-defined grid of hyperparameter values and selecting the configuration that minimizes the BIC score:
$$
\boldsymbol{\gamma}^*=\arg \min _{\boldsymbol{\gamma}} \operatorname{BIC}(\boldsymbol{\gamma}).
$$
\section{Theoretical Analysis}
\subsection{Identifiability of Co-hub Decomposition}
In this paper, we decompose the graph Laplacians as $\mathbf{L}^k=\mathbf{S}^k+\left(\mathbf{V}+\mathbf{V}^{\top}\right)$. It is important to determine whether the decomposition can uniquely identify the shared and view-specific structures. Let $\mathbb{H}=\{(u, v) \mid u$ or $v$ is a co-hub node $\}$ be the set of co-hub edges. Suppose that we have two different decompositions of $\left\{\mathbf{L}^k\right\}_{k=1}^K$ as
$
\mathbf{L}^k=\mathbf{S}_{(i)}^k+\mathbf{H}_{(i)}, \quad \mathbf{H}_{(i)}=\mathbf{V}_{(i)}+\mathbf{V}_{(i)}^{\top}$, for $i\in\{1,2\}.
$
The following theorem addresses the issue of  identifiability of the decomposition of the Laplacian matrices for each view.
\begin{theorem}
    Under the assumption that $ \operatorname{supp}\left(\mathbf{H}_{(i)}\right) \subseteq \mathbb{H}$ for $i=1,2$, any two decompositions must satisfy:
    \begin{itemize}
        \item[1.] For every $(u, v) \notin \mathbb{H}$ and every $k\in\{ 1,2,\ldots,K\}$,
$$
\left[\mathbf{S}_{(1)}^k\right]_{u v}=\left[\mathbf{S}_{(2)}^k\right]_{u v} .
$$
\item[2.] There exists a symmetric matrix $\mathbf{A}$ with $\operatorname{supp}(\mathbf{A}) \subseteq \mathbb{H}$ such that for $(u,v) \in \mathbb{H}$ and  $k\in \{1,2,\ldots,K\}$,
$$
\left[\mathbf{S}_{(1)}^k\right]_{uv}=\left[\mathbf{S}_{(2)}^k\right]_{uv}+\mathbf{A}_{uv}, \ \left[\mathbf{H}_{(1)}\right]_{uv}=\left[\mathbf{H}_{(2)}\right]_{uv}-\mathbf{A}_{uv}.
$$ 
\end{itemize}
 \end{theorem}
\noindent 
The theorem states that all view-specific edges outside the co-hub edge set $\mathbb{H}$ are completely identifiable. In the set of co-hub edges $\mathbb{H}$, the theorem states that there is a shift of a single symmetric matrix $\mathbf{A}$ (supported on $\mathbb{H}$ ) from the shared part $\mathbf{V}+\mathbf{V}^{\top}$ to every $\mathbf{S}^k$ simultaneously. If we have much stronger assumption such as $\operatorname{supp}\left(\mathbf{S}^k\right) \cap \mathbb{H}$ is null set for all $k$, then all the edges are fully identifiable. But in practice, with appropriate weight for example, $\gamma_4$ being not too small as compared to $\gamma_3$, the optimizer tends to strore edges across the views in $\mathbf{V}$ rather in putting them in $\mathbf{S}^k$. The proof of this theorem is provided in the Appendix \ref{sec: Identifiability Proof}. 

\subsection{Convergence of the Algorithm}
We solve the optimization problem in (\ref{eq:objFunc11}) using a four block ADMM algorithm. We can write the problem as 
\begin{equation}
\min _{\mathbf{x}_1, \mathbf{x}_2, \mathbf{x}_3, \mathbf{x}_4}\  \sum_{i=1}^{4}f_{i}(\mathbf{x}_{i})\quad \text { s.t. } \quad \sum_{i=1}^{4}\mathbf{A}_i \mathbf{x}_i= \mathbf{0},\end{equation}
where $\mathbf{x}_i$ and $\mathbf{A}_i$,  for $i=1,2,3,4$, are the variables in a vectorized form corresponding to the four blocks and the coefficient matrices, respectively. We can stack the Lagrangians in a vector and denote it as $\boldsymbol{\lambda}$. 
Using \cite{lin2015sublinearconvergenceratemultiblock}, we  state the following theorem about the convergence of the algorithm.   
\begin{theorem}
    Suppose the ADMM iterates are generated as  $\left\{\left(\mathbf{x}_1^{k+1}, \mathbf{x}_2^{k+1}, \mathbf{x}_3^{k+1}, \mathbf{x}_4^{k+1}, \boldsymbol{\lambda}^{k+1}\right)\right\}$ cyclically over the four blocks. For $t \geq 0$, define the ergodic averages $\bar{\mathbf{x}}_i^t=\frac{1}{t+1} \sum_{k=0}^t \mathbf{x}_i^{k+1}$ and $\bar{\boldsymbol{\lambda}}^t=\frac{1}{t+1} \sum_{k=0}^t \boldsymbol{\lambda}^{k+1}$. Then for $\alpha \leq \alpha_M,$ there exists a saddle point $\left(\mathbf{x}_1^*, \mathbf{x}_2^*, \mathbf{x}_3^*, \mathbf{x}_4^*, \boldsymbol{\lambda}^*\right)$ such that \begin{align*}
        & \,\,\,\,\,\,\,\,\ \left|f\left(\overline{\mathbf{x}}^t\right)-f\left(\mathbf{x}^*\right)\right|=\mathcal{O}(1 / t),\\
          &\left\|\mathbf{A}_1 \overline{\mathbf{x}}_1^t+\mathbf{A}_2 \overline{\mathbf{x}}_2^t+\mathbf{A}_3 \overline{\mathbf{x}}_3^t+\mathbf{A}_4 \overline{\mathbf{x}}_4^t\right\|_{F}=\mathcal{O}(1 / t), 
\end{align*}
where  $f(\mathbf{u})= \sum_{i=1}^{4}f_{i}(\mathbf{u}_i)$, $\overline{\mathbf{x}}= \left(\overline{\mathbf{x}}_1^{\top},\overline{\mathbf{x}}_2^{\top},\overline{\mathbf{x}}_3^{\top},\overline{\mathbf{x}}_4^{\top}\right)^{\top}$ and $\mathbf{x}^{*}= \left({\mathbf{x}_1^*}^{\top}, {\mathbf{x}_2^*}^{\top}, {\mathbf{x}_3^*}^{\top}, {\mathbf{x}_4^*}^{\top}\right)^{\top}$.   
\end{theorem}
\noindent  Here the convergence is global with a sublinear rate \cite{lin2015sublinearconvergenceratemultiblock}. The details of the coefficient matrices, the upper bound $\alpha_{M}$ and the proof of the theorem can be found in the appendix \ref{convergence_proof} .    

\subsection{Estimation Error Bound}
To facilitate a unified analysis of multiple graph Laplacians, we consider the parameter space as the set of block diagonal matrices in \( \mathbb{R}^{(2K+1)n \times (2K+1)n} \), where each block corresponds to the graph Laplacians $ \mathbf{L}^1 ,  \mathbf{L}^2 , \ldots , \mathbf{L}^K$, $\mathbf{S}^{1},\mathbf{S}^{2},\ldots,\mathbf{S}^{K}$  and the matrix \( \mathbf{V} \), i.e.,
$
\operatorname{bldiag}\left(\mathbf{L}^1 ,  \mathbf{L}^2 , \ldots , \mathbf{L}^K, \mathbf{S}^{1},\mathbf{S}^{2},\ldots,\mathbf{S}^{K}, \mathbf{V} \right).
$ For simplicity, assume \( d_1 = d_2 = \ldots = d_K =d \) and consider the rescaled version of the objective function as,
\begin{align}
\label{rescale_optimization}
    & \sum^K_{k=1} \left[\frac{1}{d} \operatorname{tr}\left( {\bold{X}^{k}}^\top\bold{L}^{k} {\bold{X}^{k}}\right) + \gamma_{1d} \left\| \bold{L}^{k}-\bold{I}\odot\bold{L}^{k
    }\right\|_F^2- \gamma_{2d} \operatorname{tr} \left(\log\left(\bold{I}\odot\bold{L}^{k}\right)\right)+ \gamma_{4d} \left\|\bold{S}^k\right\|_{F}^2\right]   +\gamma_{3d} \left\| \bold{V}\right\|_{2,1},
\end{align}
 where $\gamma_{jd}=\gamma_{j}/d$ for $j=1,2,3,4$ subject to the constraints,
$\bold{L}^{k}  \in \mathbb{L}=\left\{\mathbf{L} \in \mathbb{R}^{n \times n}: \mathbf{L} \succeq 0, L_{ij}=L_{ji} \leq 0, \mathbf{L} \cdot \mathbf{1}=0\right\}$
and 
$\bold{L}^{k},\bold{S}^{k}, \mathbf{V} \in \mathbb{C}= \{ \mathbf{L}^{k}, \mathbf{S}^{k}, \mathbf{V} \in \mathbb{R}^{n \times n} : \mathbf{L}^{k}- \mathbf{S}^{k} = \mathbf{V} + \mathbf{V}^{\top}\}.
$
 The block diagonals corresponding to the estimated \( \widehat{\mathbf{L}}_{\boldsymbol{\gamma}} \) and true \( \mathbf{L}^{*} \) are respectively  \ $ \operatorname{bldiag}\left(\widehat{\mathbf{L}}^{1}_{\boldsymbol{\gamma}},\widehat{\mathbf{L}}^{2}_{\boldsymbol{\gamma}}, \ldots, \widehat{\mathbf{L}}^{K}_{\boldsymbol{\gamma}},\widehat{\mathbf{S}}^{1}_{\boldsymbol{\gamma}},\widehat{\mathbf{S}}^{2}_{\boldsymbol{\gamma}}, \ldots , \widehat{\mathbf{S}}^{K}_{\boldsymbol{\gamma}},\widehat{\bold{V}}_{\boldsymbol{\gamma}}\right)$ and \\ $\operatorname{bldiag}\left(\mathbf{L}^{1^*},\mathbf{L}^{2^*}, \ldots , \mathbf{L}^{K^*},\mathbf{S}^{1^*},\mathbf{S}^{2^*}, \ldots , \mathbf{S}^{K^*},\bold{V^*}\right).$ To derive the theoretical result on estimation error bound, we introduce the following assumptions.
\begin{itemize}
    \item[(A1)] {(\textit{Sub-Gaussian Signal Assumption})} The set of signals \( \left\{\bold{X}_{\cdot j}^{k}\right\}_{j=1}^{d} \) follows an i.i.d. sub-Gaussian distribution   with mean \( \boldsymbol{0} \) and covariance matrix \( \mathbf{\Sigma}^{k^*} \) for \( k=1,2,\ldots, K \).

\item[(A2)] For each $k\in\{1,2,\ldots K\},$ the diagonal entries are strictly positive at both the end points and along the line segment, i.e.,  
$$
{L_{i i}^{k}}^{*}>0, \quad \left[\widehat{L}_{\boldsymbol{\gamma}}^{k}\right]_{ii}>0,  \text { and } \operatorname{diag}\left(\mathbf{L}^{(k)}(\tau)\right) \in(0, \infty)^n,
$$
for $\mathbf{L}^{k}(\tau):={\mathbf{L}^{k}}^{*}+\tau\left(\widehat{\mathbf{L}}_{\boldsymbol{\gamma}}^{k}-{\mathbf{L}^{k}}^{*}\right)$ and $\tau \in \left[0,1\right]$. 
\item[(A3)] {(\textit{Curvature Control in Diagonals})} For every $k\in\{1,2,\ldots K\},$ there exists $0< M_k<\infty$ such that
$$
\sup _{\tau \in[0,1]} \max _{1 \leq i \leq n}\left\{{L_{i i}^{k}}^{*}+\tau\left(\left[\widehat{L}_{\gamma}^{k}\right]_{i i}-{L_{i i}^{k}}^{*}\right)\right\} \leq M_k .
$$
\end{itemize}

\begin{theorem}
\label{th:estimation error}
    Under assumptions (A1), (A2), and (A3), with the regularization parameters \( \gamma_{1d},\gamma_{2d},\gamma_{3d}, \gamma_{4d} > 0 \), the estimation error $\left\|\mathbf{\Delta}\right\|_{F}=\left\|\widehat{\bold{L}}_{\boldsymbol{\gamma}}- \bold{L}^{*}\right\|_{F}$ will be upper bounded as follows, 

\begin{align}
       \left\|\mathbf{\Delta}\right\|_{F}\leq \frac{2K n}{\mu\sqrt{d}} \widetilde{C}+\frac{2}{\mu}\left(C^{\prime} \sqrt{K}+\gamma_{3 d} \sqrt{h}+\sqrt{\frac{ \mu\gamma_{4 d}C_{S^*}}{2}}\right)\end{align}
with probability at least $ 1 - 2K\,e^{-c^k a^2 n}$ with the constant $a \geq \sqrt{\frac{\log (2 K)}{c^k n}}$, where  $c^k$ and $C^k$ are constants that depend on the sub-Gaussian norms $\max_{i}\left\|\mathbf{X}_{\cdot i}^{k}\right\|_{\psi}$ of a random vector
taken from this distribution, $\widetilde{C}^{k}=\operatorname{max}\left\{C^{1},C^{2},\ldots,C^{K}\right\}$, $C^{\prime}= \max _{1 \leq k \leq K} C_k$ with $C_k:=\left\|{\mathbf{\Sigma}^k}^{*}\right\|_F+2 \gamma_{1 d}\left\|{\mathbf{L}^k}^{*}-\mathbf{I} \odot {\mathbf{L}^k}^{*}\right\|_F+\gamma_{2 d}\left\|\left(\mathbf{I} \odot {\mathbf{L}^k}^{*}\right)^{-1}\right\|_F$,  $C_{S^*} = \sum_{k=1}^K\left\|\mathbf{S}^{k^ *}\right\|_F^2$ and $\mu= \min_{1\leq k\leq K}\min \left\{2 \gamma_{1 d}, \frac{\gamma_{2 d}}{M_k^2}\right\}.$
\end{theorem}
\noindent This bound has two parts. The first term is of $\mathcal{O}(1/\sqrt{d})$ and captures the classical sampling error. As the number of samples increases the error decreases  as expected. The second term represents the bias introduced by regularization.  It contains a $\gamma_{3 d} \sqrt{h}$ factor that scales with the square-root of the true number of co-hub nodes rather than with $n$. This $\sqrt{h}$ dependence  highlights how the $\ell_{2,1}$ penalty exploits the shared‐hub structure.he proof of this
theorem is provided in the Appendix \ref{proof_estimation error}. 

\section{Experimental Results}
\subsection{Simulated Data}

In this paper, we consider three random network models: Erd\H{o}s-R\'{e}nyi (ER) random network, Barab\'{a}si-Albert model (BA) and random geometric graph (RGG). In ER graphs, the pairs of nodes are independently connected with probability $0.1$. For the BA model, each new node connects to $m$ existing nodes that already have more links. Over time, this leads to a few nodes getting many connections, creating highly linked nodes in the network. For RGG, we used the setup from \cite{kalofolias2016learn}, where 100 two-dimensional points are randomly drawn from $[0, 1]^2$ and they are connected to each other with weights $ \exp(-\left\|\mathbf{x}_i - \mathbf{x}_j\right\|_2^2/\sigma^2)$ where $\mathbf{x}_{i}$ is the coordinates of $i$th point and $\sigma=0.25$. Weights smaller than $0.6$ are set to $0$, while the remaining ones are set to $1$ to generate binary graphs. For all network models, we randomly select $h$ nodes as co-hub nodes. For each selected co-hub node, we set the elements of the corresponding row and column of each $\left\{\mathbf{A}^k\right\}_{k=1}^K$ to be i.i.d. from a Bernoulli distribution. This results in $h$ co-hub nodes.
\noindent \paragraph{Data Generation}
Given $K$ views, each $\mathbf{X}^k \in \mathbb{R}^{n \times d_{k}}$  is generated from $G^k$ using the smooth graph filter $h(\mathbf{L}^k)$ \cite{kalofolias2016learn}. In particular, each column of $\mathbf{X}^k$ is generated as $\mathbf{X}_{\cdot j}^k = h\left(\mathbf{L}^k\right) \mathbf{x}_0$; where $\mathbf x_0 \sim \mathcal{N}(\zeros, \mathbf{I})$. In this paper, we consider three different graph filters: 1)  Gaussian filter $\left[h\left(\mathbf{L}\right) = \mathbf{L}^\dagger\right]$; 2) Heat filter $\left[h(\mathbf{L}) = \exp\left(-\alpha \mathbf{L}\right)\ \text{with}\ \alpha=5\right]$; and 3) Tikhonov filter $\left[h(\mathbf{L}) = \left(\mathbf I + \alpha\mathbf{L}\right)^{-1}\ \text{with}\ \alpha = 20\right]$. In the case of the Gaussian filter, the resulting signals are Gaussian distributed and the graph Laplacian and the precision matrix are equivalent to each other. We finally add $\eta\%=10\%$ noise (in $\vell_2$ norm sense) to $\mathbf{X}^k$.
\paragraph{Benchmark Models}
We compare the proposed method with respect to single view graph learning (SV) \cite{dong2016learning} that learns the graph topology for each view independently by assuming that the signals are smooth with respect to the graph and co-hub node joint graph learning (CNJGL) \cite{mohan2014node} that jointly learns precision matrices with co-hub nodes assuming a Gaussian Graphical Model (GGM). The performance of all methods is quantified by computing the average F1 score with respect to the ground truth graphs across $50$ realizations.
\paragraph{Results and Discussion}
In the first experiment, we evaluate the performance of  CH-MVGL with respect to the number of views, percentage of co-hub nodes, and noise level using ER model with heat filter. First, we set $n=256$, the co-hub node percentage to $3\%$, and the noise level to $10\%$, while varying the number of views from $2$ to $14$. Next, we set $K=6$, $n=256$, and the noise level to $10\%$, while varying the percentage of co-hub nodes from $3\%$ to $12\%$ in increments of $3\%$. Finally, we set $K=6$, $n=256$, and the co-hub node percentage to $3\%$, while varying the noise level from $10\%$ to $70\%$ in increments of $20\%$. In Fig. \ref{fig:res views}, it  can be seen that the performance of CH-MVGL and CNJGL  improve with increasing number of views as  more views improve learning accuracy, while CH-MVGL achieves the best results. All methods outperform SV because it fails to leverage the similarity across views. From Fig. \ref{fig:res cohub percantges}, when the percentage of co-hub nodes increases, the performance of all methods decreases as $\mathbf{V}$ is no longer sparse, making it difficult to detect the correct co-hub nodes. Finally, the performance of all methods decreases as the noise level is increased. Although all of the methods are robust against noise, SV outperforms CNJGL since  CNJGL is designed to estimate precision matrices and performs worse on non-Gaussian data as shown in  Fig. \ref{fig:res noise level percantges}. CH-MVGL achieves the best performance compared to the other methods. 

In the second experiment, we evaluated the performance of CH-MVGL with respect to three different graph models, i.e., ER, BA and RGG, for three different graph filters, i.e., Gaussian, Heat and Tikhonov.  In this experiment, we set the number of views to $K = 6$, the number of nodes to $n = 128$, and the percentage of co-hub nodes to $3\%$. As shown in Fig. \ref{fig:simulated ER BA RGG}, CH-MVGL outperforms the other methods for all graph and signal models. Among the compared methods, CNJGL performs closest to CH-MVGL when the signal is Gaussian, exhibiting the smallest variance due to its focus on learning precision matrices. However, when the signal is smooth, CNJGL's performance declines significantly compared to CH-MVGL, and its variance increases substantially.  CH-MVGL performs the best for the BA model as the inherent hub structure aligns well with the underlying assumptions.

\begin{figure*}
    \centering
    
    \begin{subfigure}{0.28\linewidth}
        \centering
        \includegraphics[width=\linewidth]{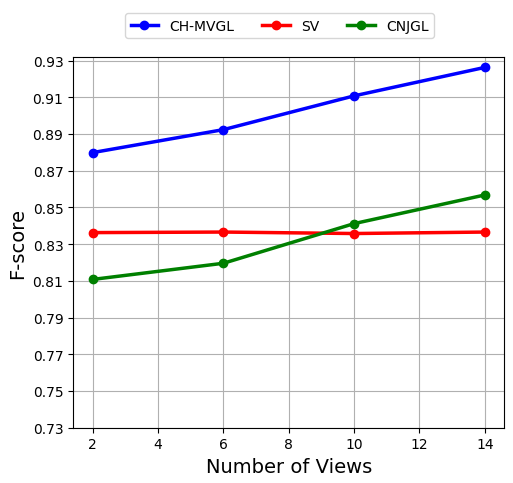}
        \caption{}
        \label{fig:res views}
    \end{subfigure}
    \begin{subfigure}{0.28\linewidth}
        \centering
        \includegraphics[width=\linewidth]{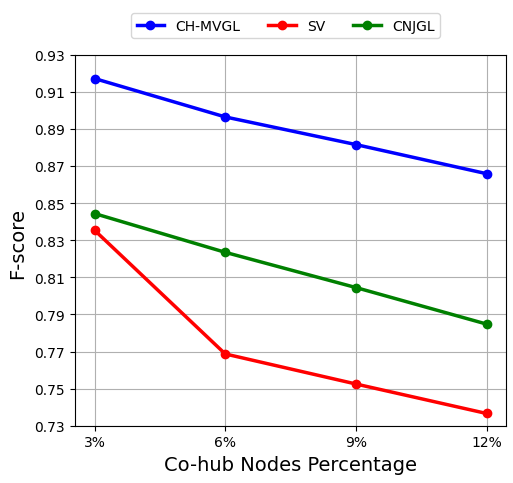}
        \caption{}
        \label{fig:res cohub percantges}
    \end{subfigure}
    \begin{subfigure}{0.28\linewidth}
        \centering
        \includegraphics[width=\linewidth]{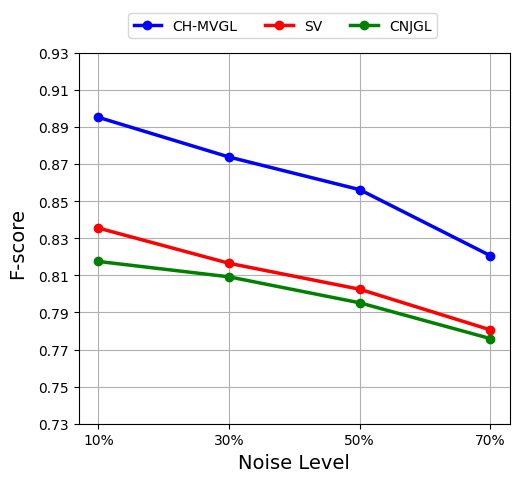}
        \caption{}
        \label{fig:res noise level percantges}
    \end{subfigure}
     
    \caption{Comparison of performance for ER network model with heat graph filter with respect to (a) Number of views, (b) Number of co-hub nodes and (c) Noise level.}
    \label{fig:sca}
    \vspace{-0.4cm}
\end{figure*}
\hspace{1pt}
\begin{figure*}
    \centering
    
    \begin{subfigure}{0.28\linewidth}
        \centering
        \includegraphics[width=\linewidth]{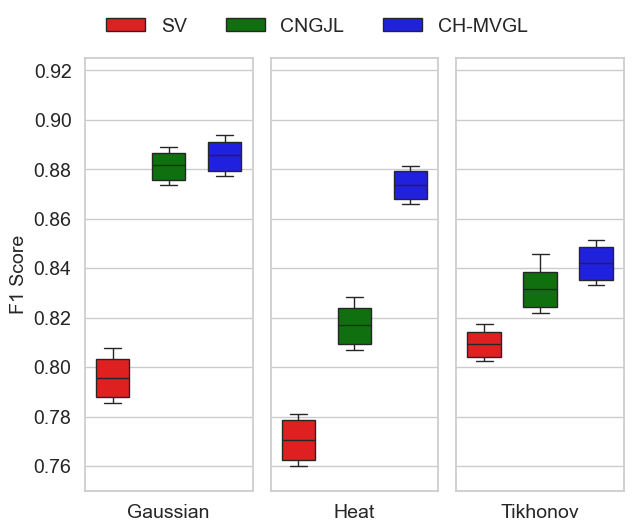}
        \caption{}
        \label{}
    \end{subfigure}
    \begin{subfigure}{0.28\linewidth}
        \centering
        \includegraphics[width=\linewidth]{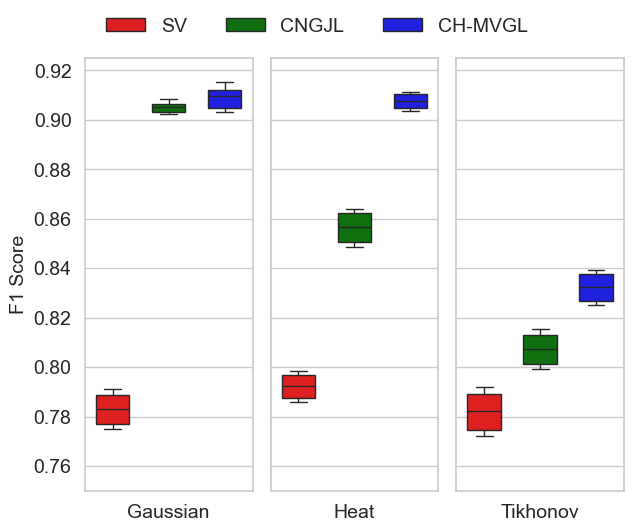}
        \caption{}
        \label{}
    \end{subfigure}
    \begin{subfigure}{0.28\linewidth}
        \centering
        \includegraphics[width=\linewidth]{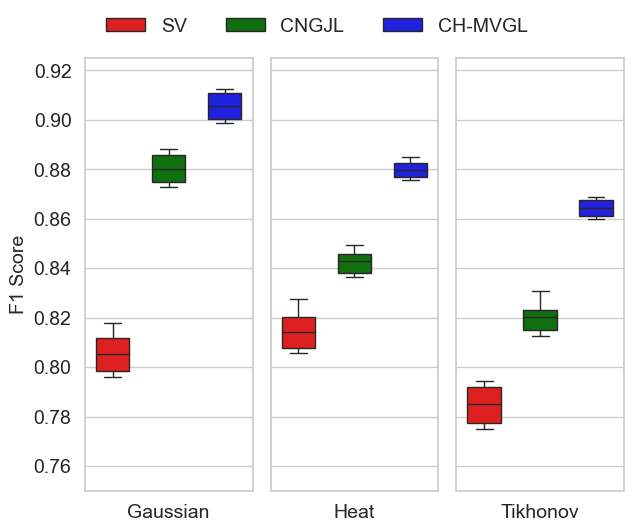}
        \caption{}
        \label{}
    \end{subfigure}
     
   \caption{Comparison of CH-MVGL to SV and CNGJL for
different graph filters for (a) ER, (b) BA, and (c) RGG random network models}
     \label{fig:simulated ER BA RGG}
    \vspace{-0.4cm}
\end{figure*}

\paragraph{Scalability Analysis}
The scalability of the proposed method with respect to the number of nodes and views is evaluated for ER multiview graphs with $h=0.02n$ and $d_{k}=700$ for all views.  In the first experiment, the number of nodes is increased logarithmically from $64$ to $2048$, and the number of views is set to $6$. In the second experiment, the number of views increases from $2$ to $14$, and the number of nodes is set to $256$.  Fig. \ref{fig:sca nodes} and Fig. \ref{fig:sca views} illustrate the run times of the proposed method and the benchmarking techniques for the first and second experiments, respectively. The run time of all methods increases with the number of views and nodes, with the number of nodes affecting the complexity more. SV is the fastest method as it learns each view separately without any regularization. CNJGL is the slowest as it performs SVD in each iteration. The computational complexity of CH-MVGL is $\mathcal{O}(Kn^3)$, so increasing the number of views increases the complexity linearly, while increasing the number of nodes increases it cubically.

\begin{figure*}
    \centering
    
    \begin{subfigure}{0.30\linewidth}
        \centering
        \includegraphics[width=\linewidth]{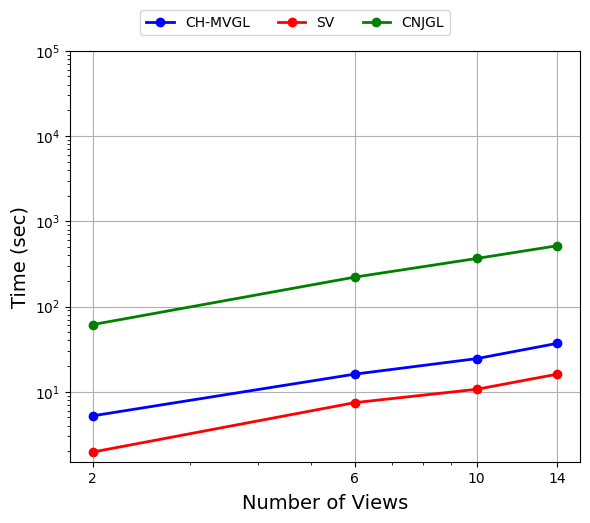}
        \caption{}
        \label{fig:sca nodes}
    \end{subfigure}
    \begin{subfigure}{0.30\linewidth}
        \centering
        \includegraphics[width=\linewidth]{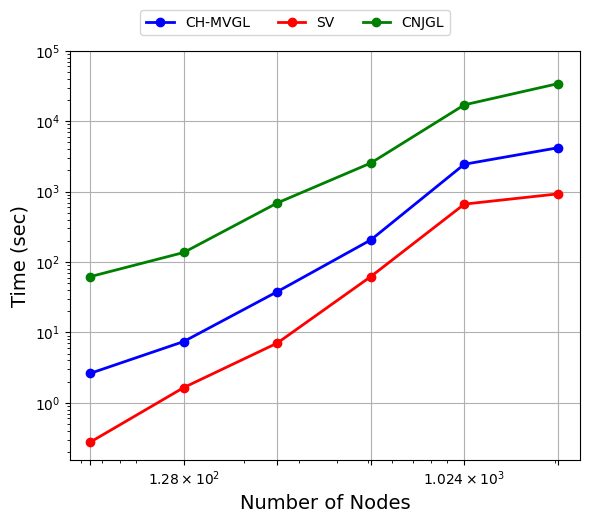}
        \caption{}
        \label{fig:sca views}
    \end{subfigure}
     
\caption{Scalability of CH-MVGL compared to existing methods
with respect to: (a) Number of views; (b) Number of nodes.}
     \label{fig:sca}
     \vspace{-0.4cm}
\end{figure*}
\subsection{fMRI Data Analysis}
Hub node identification is an important problem in neuroimaging as hub nodes correspond to regions relevant for different cognitive processes \cite{xu2022meta}. 
Traditional hub node detection methods rely on functional connectivity networks (FCNs) constructed from fMRI time series using Pearson's correlation. The resulting graphs are usually fully connected and weighted and do not necessarily capture the topological structure of brain networks  \cite{gao2021smooth}. In this paper, we implement CH-MVGL to jointly learn the FCNs from fMRI time series while simultaneously identifying the co-hubs across subjects with the assumption that the FCNs across subjects in a healthy population and for a given task have common hub nodes \cite{ortiz2025learning}. 

CH-MVGL is applied to functional neuroimaging data from $55$ subjects from the Human Connectome Project (HCP)\footnote{Details of data acquisition can be found at \url{db.humanconnectome.org}}. One hour of resting state data was acquired per subject in 15-minute intervals over two separate sessions with eyes open and fixation on a crosshair.  Functional volumes are spatially smoothed with a Gaussian kernel (5 mm full-width at half-maximum). 
The first 10 volumes are discarded so that the fMRI signal achieves steady-state magnetization, resulting in $1190$ time points. Voxel fMRI time series are detrended and band-pass filtered [0.01 - 0.15] Hz 
to improve the signal-to-noise ratio for typical resting-state fluctuations. Finally, Glasser’s multimodal parcellation \cite{glasser2016multi} resliced to fMRI resolution is used to parcellate fMRI volumes and compute regionally averaged signals which 
 are used as the graph signals. 

The proposed objective function in Eq. \eqref{eq:objFunc} identifies co-hubs as the non-zero columns in the matrix $\mathbf{V}$. The norm of the corresponding column quantifies the strength of each co-hub. To determine the number of co-hubs, each non-zero column is normalized by the maximum norm of the columns of $\mathbf{V}$ and sorted in descending order. Using the resting state data from $55$ subjects from sessions $1$ and $2$ separately, we plotted the normalized norm of the columns of $\mathbf{V}$ as shown in Fig. \ref{fig:cohubs_number}. The curves corresponding to the two scans show a sharp initial drop followed by a flattening trend. 
Based on this curve, the number of co-hub nodes is determined to be $6$ for both sessions.


The spatial distribution and connectivity of co-hubs is illustrated in Fig. \ref{fig:ccohubs onnectivity} where the top $6$ co-hubs identified across $55$ subjects for sessions $1$ and $2$ are shown. For session $1$, the co-hub nodes are located in the default mode network (DMN). For session $2$, the majority of the hub nodes are located in the default mode networks, while one is in the dorsal attention networks. Moreover,  a couple of the co-hubs have significantly high number of long-range connections across brain regions and hemispheres. These findings are consistent with prior studies on resting-state fMRI \cite{cole2010identifying,tomasi2011association,de2013connectivity,xu2022meta} which showed that the majority of the hub nodes are in the DMN followed by dorsal attention networks. In these studies, DMN has been shown to have the highest global brain connectivity, which may reflect the connections necessary to implement the wide variety of cognitive functions in which this network is involved.

Furthermore, we evaluated the replicability of the detected hubs in sessions $1$ and $2$. For this analysis, we followed a replacement approach by resampling, where $40$ subjects were randomly selected from the total of $55$. The frequency histograms of the selected hub nodes and the corresponding entropy values are computed to determine the consistency of the selected co-hub nodes in each run. 
Lower entropy values suggest that specific nodes are consistently selected as co-hubs, indicating higher replicability. Fig. \ref{fig:Replicability} shows the frequency histograms and the entropy for both sessions along with the selected hub nodes across runs. In the first session, $4$ hub nodes are identified in every resampling run, while $2$ additional hubs are identified with slightly lower frequency. This implies a stable set of co-hubs, alongside some variability across resamples. In session $2$, $2$ hub nodes are identified in almost every resampling run, while $2$ additional hubs are identified with slightly lower frequency. Both sessions identified co-hub nodes consistent with previous studies, showing low entropy that indicates a stable pattern of co-hubs. The second session exhibited slightly lower entropy than the first, suggesting a marginally more stable configuration.

\begin{figure*}
    \centering
    
    \begin{subfigure}{0.30\linewidth}
        \centering
        \includegraphics[width=\linewidth]{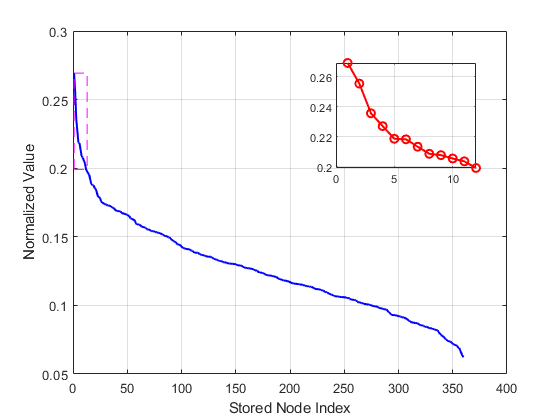}
        \caption{}
        \label{}
    \end{subfigure}
    \begin{subfigure}{0.30\linewidth}
        \centering
        \includegraphics[width=\linewidth]{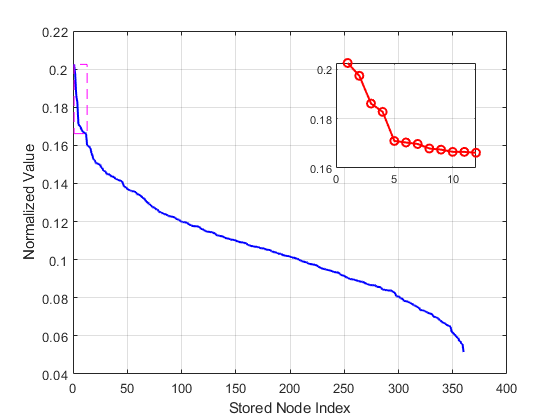}
        \caption{}
        \label{}
    \end{subfigure}
        \caption{Normalized values of nodes used to determine co-hub nodes for (a) session $1$, and (b) session $2$.}
    \label{fig:cohubs_number}
 \end{figure*}    

\begin{figure}
    \centering
        \includegraphics[width=0.45\linewidth]{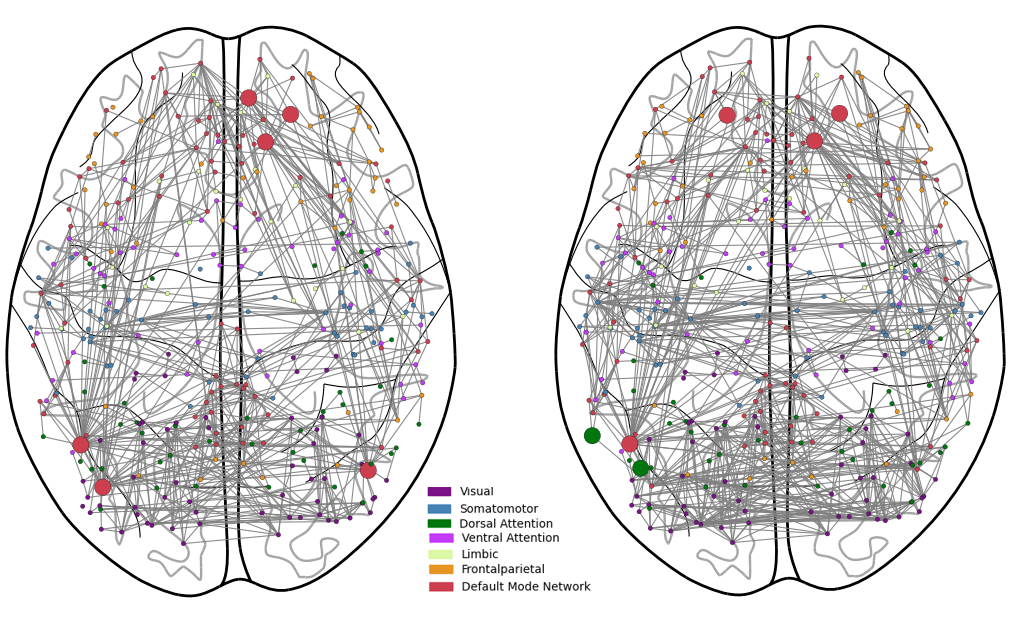}
    
    \caption{Connectivity of co-hub nodes identified by the proposed CH-MVGL model on the HCP dataset for sessions $1$ and $2$. The left panel corresponds to session $1$, and the right panel to session $2$. The larger nodes represent the co-hubs.}
    \label{fig:ccohubs onnectivity}
\end{figure}    

\begin{figure}
    \centering
        \includegraphics[width=0.45\linewidth]{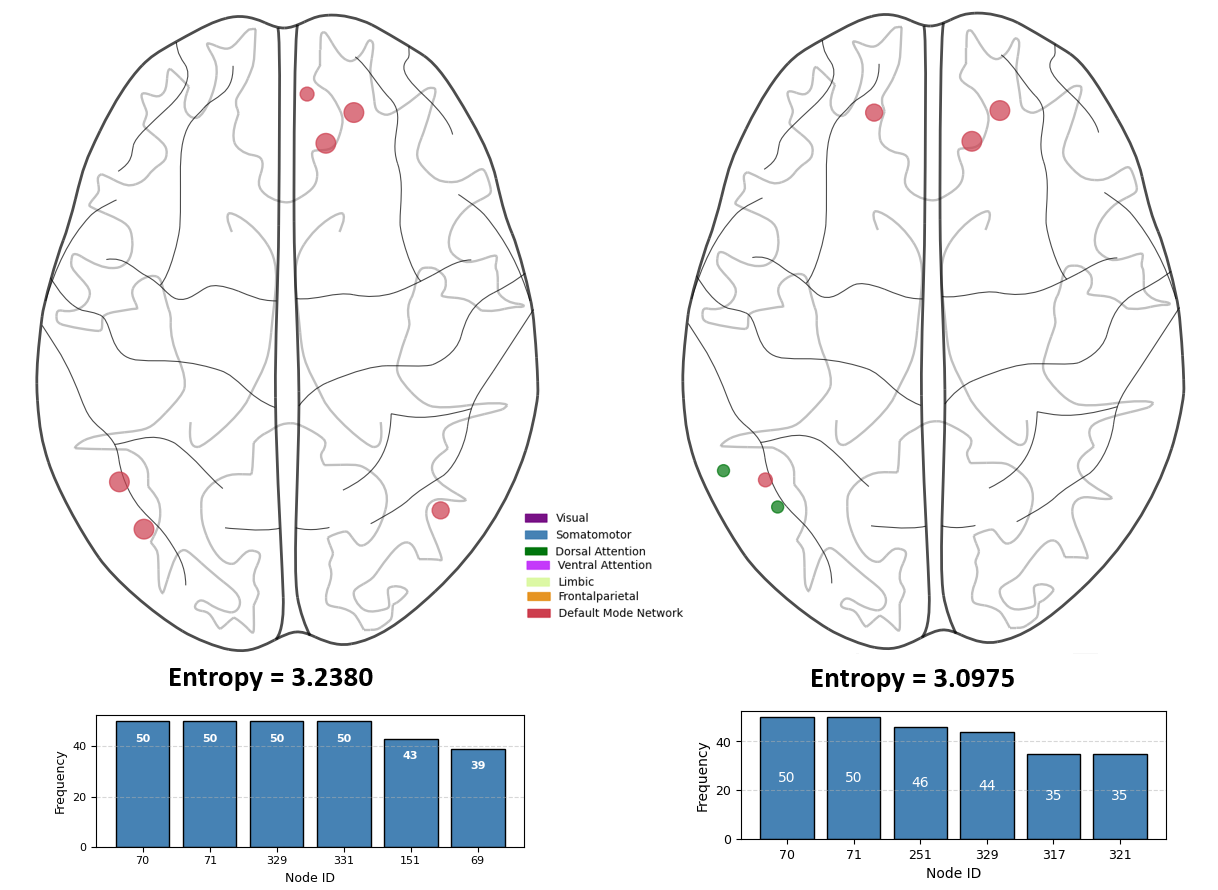}
    
    \caption{The replicability of the proposed CH-MVGL model on the HCP dataset across sessions $1$ and $2$. The left panel corresponds to session $1$, and the right panel to session $2$. The size of each node indicates how frequently it was identified as a co-hub node.}
    \label{fig:Replicability}
\end{figure}
\section{Conclusions}
In this paper, we introduce a multiview graph learning approach based on graph signal processing, in which different graphs share a common set of hub nodes. The proposed model learns graph Laplacians with the assumption that the observed signals are smooth with respect to the underlying graphs, while the different views share a common structure through their hubs.  This model presents a novel approach to joint graph learning, diverging from traditional methods that typically focus on edge-based similarity rather than node-based similarity. The uniqueness of the proposed model is proven, showing that all view-specific edges outside the co-hub edge set are identifiable. We also provide an estimation error bound between the estimated and true Laplacian matrices, showing that while the error decreases as the number of samples approaches infinity, a nonzero bias remains due to the number of co-hub nodes.
The proposed model is applied to simulated data generated from various random graph models and graph filters, yielding distinct signal models. The proposed method is also used to analyze resting-state fMRI data across multiple subjects. The proposed method primarily detects hub nodes in the default mode network, consistent with prior studies on resting-state networks. While the current model assumes that the hub nodes are common across all views, it can be easily modified to accommodate common and private hub nodes across views, similar to \cite{kim2019bayesian,huang2024njgcg}. Future work will also consider extensions to dynamic graphs with temporal smoothness and the incorporation of other graph structures, such as communities, in multiview graph learning.

\section{Appendix}
\subsection{Derivation of the update steps}
\label{sec:Appendix_derivations}

The solution to Eq. \eqref{eq:LagFunc} can be obtained by breaking it down into subproblems and solving each one individually.

    \noindent \textbf{$\mathbf{E}^k$ update}: $\mathbf{E}^k$ subproblem can be written as follows:
    \begin{equation}
        \begin{split}
            &\min_{\substack{\mathbf E^{k}_l}}\trace(\mathbf{B}^k \mathbf{E}^k_l) + \frac{\alpha}{2}\left\|\mathbf{\Psi}_l^k-\mathbf{P}\mathbf{E}^k_{l}{\mathbf{P}}^\top+\mathbf{Z}_l^k+\frac{\mathbf{R}_l^k}{\alpha}\right\|_F^2 +\frac{\alpha}{2}\left\Vert\mathbf{C}^k_l-\mathbf P \mathbf{E}^k_l \mathbf P^\top+\frac{\mathbf{Y}^k_l}{\alpha}\right\Vert_F^2.
        \label{eq: Ek sub}
        \end{split}
    \end{equation}
The solution of Eq. \eqref{eq: Ek sub} can be found by taking the gradient and setting it to zero, which yields:
\begin{equation}
    \begin{split}
        \mathbf{E}^k_{l+1}=\frac{\alpha\mathbf{P}^\top\mathbf{C}^k_l\mathbf{P}+\alpha\mathbf{P}^\top\mathbf{\Psi}^k_l\mathbf{P}+\alpha\mathbf{P}^\top\mathbf{Z}^k_l\mathbf{P}+\mathbf{\theta}^k_l}{2\alpha},
        \label{eq: Ek sol}
    \end{split}
\end{equation}
\noindent where $\mathbf{\theta}^k_l=\mathbf{P}^\top\mathbf{R}^k_l\mathbf{P}+\mathbf{P}^\top\mathbf{Y}^k_l\mathbf{P}-{\mathbf{B}^k}^\top$.
\\
\noindent \textbf{$\mathbf{\Xi}^k$ update}: $\mathbf{\Xi}^k$ subproblem can be written as follows:
\begin{equation}
    \begin{split}
    &\min_{\substack{\mathbf \Xi_l^{k}}} \gamma_4\left\|\mathbf P \mathbf{\Xi}^k_l \mathbf P^\top\right\|_F^2 +\frac{\alpha}{2}\left\Vert\mathbf{\Gamma}^k_l-\mathbf P \mathbf{\Xi}^k_l \mathbf P^\top+\frac{\mathbf{T}^k_l}{\alpha}\right\|_F^2.
        \label{eq: Xi sub}
    \end{split}
\end{equation}
Similar to Eq. \eqref{eq: Ek sub}, the solution of Eq. \eqref{eq: Xi sub} can be found as follows:
\begin{equation}
    \begin{split}
        \mathbf{\Xi}^k_{l+1}=\frac{\alpha\mathbf{P}^\top\mathbf{\Gamma}^k_l\mathbf{P}+\mathbf{P}^\top\mathbf{T}^k_l\mathbf{P}}{2\gamma_4+\alpha}.
        \label{eq: Xi sol}
    \end{split}
\end{equation}

\noindent \textbf{$\mathbf{V}$ update}: $\mathbf{V}$ subproblem can be written as follows:
\begin{equation}
{\begin{split}
    &\min_{\substack{\mathbf{V}_l}}  \frac{\alpha}{2}\sum_{k=1}^{K}\left\{\left\Vert\mathbf{C}^k_{l}-\mathbf{\Gamma}^k_{l}-\mathbf{V}_{l}-\mathbf{W}_{l}+\frac{\mathbf{M}^k_l}{\alpha}\right\Vert_F^2\right\}+\frac{\alpha}{2}\left\Vert\mathbf{V}_l-\mathbf{W}^\top_{l}+\frac{\mathbf{N}_l}{\alpha}\right\Vert_F^2+\frac{\alpha}{2}\left\Vert\mathbf{G}_{l}-\mathbf{V}_l+\frac{\mathbf{Q}_l}{\alpha}\right\Vert_F^2.
\end{split}}
\label{eq: V sub}
\end{equation} 
The solution of Eq. \eqref{eq: V sub} can be written as follows:

\begin{equation}
    \begin{split}
        \mathbf{V}_{l+1} = \frac{\displaystyle\sum_{k=1}^K\left[\alpha\mathbf{C}^k_{l}-\alpha\mathbf{\Gamma}^k_{l}-\alpha\mathbf{W}_{l}+\mathbf{M}^k_{l}\right]+\mathbf{\Theta}_l}{\alpha\left(K+2\right)},
    \label{eq: V sol}
    \end{split}
\end{equation}
\noindent where $\mathbf{\Theta}_l=\alpha\mathbf{W}^\top_{l}-\mathbf{N}_l+\alpha\mathbf{G}_{l}+\mathbf{Q}_l.$\\

\noindent 
\textbf{$\mathbf{Z}^k$ update}: $\mathbf{Z}^k$ subproblem can be formulated as follows:
\begin{equation}
    \begin{split}
    &\min_{\substack{\mathbf{Z}_l^k}} - \gamma_2 \trace(\log(\mathbf{Z}^k_l)) +\frac{\alpha}{2}\left\|\mathbf{Z}^k_l-\mathbf{I}\odot\mathbf{C}^k_l+\frac{\mathbf{J}^k_l}{\alpha}\right\|_F^2 +\frac{\alpha}{2}\left\|\mathbf{\Psi}_l^k-\mathbf{P}\mathbf{E}_{l+1}^k{\mathbf{P}}^\top+\mathbf{Z}_l^k+\frac{\mathbf{R}_l^k}{\alpha}\right\|_F^2.
    \label{eq: Zk sub}
    \end{split}
\end{equation}

The solution of Eq. \eqref{eq: Zk sub} can be written as follows:

\begin{equation}
    \begin{split}
        \mathbf{Z}^k_{l+1}=\left[\frac{\mathbf{U}^k_l+\sqrt{\left(\mathbf{U}^k_l\right)^2+8\alpha\gamma_2\mathbf{I}}}{4\alpha}\right]\odot \mathbf{I},    
    \label{eq: Zk sol}
    \end{split}
\end{equation}
\noindent where $\mathbf{U}^k_l=\alpha\mathbf{I}\odot\mathbf{C}^k_l-\mathbf{J}^k_l-\alpha\mathbf{\Psi}_l^k+\alpha\mathbf{P}\mathbf{E}_{l+1}^k{\mathbf{P}}^\top-\mathbf{R}_l^k$. The power and division operations in Eq. \eqref{eq: Zk sol} are element-wise. \\
\noindent \textbf{$\mathbf{\Gamma}^k$ update}: $\mathbf{\Gamma}^k$ subproblem can be formulated as follows:

\begin{equation}
{\begin{split}
    &\min_{\substack{\mathbf{\Gamma}_l^k}} \frac{\alpha}{2}\left\Vert\mathbf{C}^k_{l}-\mathbf{\Gamma}^k_{l}-\mathbf{V}_{l+1}-\mathbf{W}_{l}+\frac{\mathbf{M}^k_l}{\alpha}\right\Vert_F^2+\ \frac{\alpha}{2}\left\Vert\mathbf{\Gamma}^k_l-\mathbf P \mathbf{\Xi}^k_{l+1} \mathbf P^\top+\frac{\mathbf{T}^k_l}{\alpha}\right\Vert_F^2. 
\end{split}}
\label{eq: Gamma sub}
\end{equation} 

The solution of Eq. \eqref{eq: Gamma sub} can be written as follows:

\begin{equation}
    \begin{split}
\mathbf{\Gamma}^k_{l+1}=\frac{\alpha\left(\mathbf{C}^k_{l}-\mathbf{V}_{l+1}-\mathbf{W}_{l}+\mathbf{P}\mathbf{\Xi}^k_{l+1}\mathbf{P}^\top\right)+\mathbf{M}^k_l-\mathbf{T}^k_l}{2\alpha}.
    \label{eq: Gamma sol}
    \end{split}
\end{equation}
\\

\noindent \textbf{$\mathbf{C}^k$ update}: $\mathbf{C}^k$ subproblem can be formulated as follows:

\begin{equation}
    \begin{split}
        &\min_{\substack{\mathbf{C}_l^k}}  \frac{\alpha}{2}\left\Vert\mathbf{C}^k_l-\mathbf P \mathbf{E}^k_{l+1} \mathbf P^\top+\frac{\mathbf{Y}^k_l}{\alpha}\right\Vert_F^2 +\frac{\alpha}{2}\left\Vert\mathbf{C}^k_l-\mathbf{\Gamma}^k_{l+1}-\mathbf{V}_{l+1}-\mathbf{W}_{l}+\frac{\mathbf{M}^k_l}{\alpha}\right\Vert_F^2+\frac{\alpha}{2}\left\Vert\mathbf{Z}^k_{l+1}-\mathbf{I}\odot\mathbf{C}^k_l+\frac{\mathbf{J}^k_l}{\alpha}\right\Vert_F^2 .
        \label{eq: Ck sub}
    \end{split}
\end{equation}

The solution of Eq. \eqref{eq: Ck sub} can be found as follows:

\begin{equation}
    \begin{split}
        \mathbf{C}^k_{l+1}=\frac{\alpha\mathbf{P}\mathbf{E}^k_{l+1}\mathbf{P}^\top-\mathbf{Y}^k_l+\alpha\mathbf{Z}^k_{l+1}+\mathbf{J}^k_l+\bar{\mathbf{\theta}}^k_l}{3\alpha},
    \label{eq: Ck sol}
    \end{split}
\end{equation}
\noindent where $\bar{\mathbf{\theta}}^k_l=\alpha\mathbf{\Gamma}^k_{l+1}+\alpha\mathbf{V}_l+\alpha\mathbf{W}_{l+1}-\mathbf{M}^k_l$.\\

\noindent \textbf{$\mathbf{\Psi}^k$ update}: $\mathbf{\Psi}^k$ subproblem can be formulated as follows:
\begin{equation}
    \begin{split}
&\min_{\substack{\mathbf \Psi_l^{k}}}  \gamma_1 \left\| \mathbf{\Psi}_l^k\right\|_F^2 +\frac{\alpha}{2}\left\Vert\mathbf{\Psi}_l^k-\mathbf{P}\mathbf{E}_{l+1}^k{\mathbf{P}}^\top+\mathbf{Z}_{l+1}^k+\frac{\mathbf{R}_l^k}{\alpha}\right\Vert_F^2, 
\end{split}
\label{eq:LagFunc}
\end{equation} 
The solution of the above problem is given by:

\begin{equation}
    \begin{split}
        \mathbf{\Psi}^k_{l+1} = \frac{\alpha\mathbf{P}\mathbf{E}_{l+1}^k{\mathbf{P}}^\top-\alpha\mathbf{Z}_{l+1}^k-\mathbf{R}_l^k}{2\gamma_1+\alpha}.
        \label{eq: Psi sol}
    \end{split}
\end{equation}

\noindent \textbf{$\mathbf{G}$ update}: $\mathbf{G}$ subproblem can be formulated as follows:
\begin{equation}
{\begin{split}
    &\min_{\substack{\mathbf{G}_l}} \gamma_3\Vert \mathbf{G}_l\Vert_{2,1}+\frac{\alpha}{2}\left\Vert\mathbf{G}_l-\mathbf{V}_{l+1}+\frac{\mathbf{Q}_l}{\alpha}\right\Vert_F^2.
\end{split}}
\label{eq: G  sub}
\end{equation} 
The solution of Eq. \eqref{eq: G sub} can be found by using the proximal algorithm for $\ell_{2,1}$-norm.
\begin{equation}
    \begin{split}
\mathbf{G}_{l+1}=\mathcal{T}_{{2,1}_{\frac{\gamma_3}{2\alpha}}} \left(\mathbf{V}_{l+1}-\frac{\mathbf{Q}_l}{\alpha}\right),
        \label{eq: G sol} 
    \end{split}
\end{equation}
\noindent where $\mathcal{T}_{2,1}$ is the proximal operator for $\ell_{2,1}$-norm. \\
\noindent \textbf{$\mathbf{W}$ update}: $\mathbf{W}$ subproblem can be written as follows:
\begin{equation}
{\begin{split}
    &\min_{\substack{\mathbf{W}_l}} \sum^K_{k=1} \left\{\frac{\alpha}{2}\left\Vert\mathbf{C}^k_{l+1}-\mathbf{\Gamma}^k_{l+1}-\mathbf{V}_{l+1}-\mathbf{W}_{l}+\frac{\mathbf{M}^k_{l}}{\alpha}\right\Vert_F^2\right\} +\frac{\alpha}{2}\left\Vert\mathbf{V}_{l+1}-\mathbf{W}^\top_l+\frac{\mathbf{N}_l}{\alpha}\right\Vert_F^2. 
\end{split}}
\label{eq: W sub}
\end{equation} 
The solution of Eq. \eqref{eq: W sub} can be written as follows:

\begin{equation}
    \begin{split}
        \mathbf{W}_{l+1}=\frac{\displaystyle\sum_{k=1}^K\left[\alpha\mathbf{C}^k_{l+1}-\alpha\mathbf{\Gamma}^k_{l+1}-\alpha\mathbf{V}_{l+1}+\mathbf{M}^k_l\right]+\mathbf{\Phi}_l}{\alpha\left(K+1\right)},
    \label{eq: W sol}
    \end{split}
\end{equation}
\noindent where $\mathbf{\Phi}_l=\alpha\mathbf{V}^\top_{l+1}+\mathbf{N}^\top_l$.\\

\noindent \textbf{Lagrangian multipliers and penalty parameter}: The Lagrangian multipliers and penalty parameter can be updated as follows:
\begin{equation}
    \begin{split}
    &\mathbf{M}^k_{l+1}=\mathbf{M}^k_l+\alpha_l\left(\mathbf{C}^k_{l+1}-\mathbf{\Gamma}^k_{l+1}-\mathbf{V}_{l+1}-\mathbf{W}_{l+1}\right),\\&
    \mathbf{Y}^k_{l+1}=\mathbf{Y}^k_l+\alpha_l\left(\mathbf{C}^k_{l+1}-\mathbf P \mathbf{E}^k_{l+1} \mathbf P^\top\right), \\&
    \mathbf{T}^k_{l+1}=\mathbf{T}^k_l+\alpha_l\left(\mathbf{\Gamma}^k_{l+1}-\mathbf P \mathbf{\Xi}^k_{l+1} \mathbf P^\top\right), \\&
    \mathbf{J}^k_{l+1}=\mathbf{J}^k_l+\alpha_l\left(\mathbf{Z}^k_{l+1}-\mathbf{I}\odot\mathbf{C}^k_{l+1}\right),\\&
     \mathbf{R}_{l+1}^k=\mathbf{R}^k_l+\alpha_l\left(\mathbf{\Psi}^k_{l+1}-\mathbf{P}\mathbf{E}^k_{l+1}\mathbf{P}^\top+\mathbf{Z}^k_{l+1}\right),\\&
     \mathbf{N}_{l+1}=\mathbf{N}_l+\alpha_l\left(\mathbf{V}_{l+1}-\mathbf{W}^\top_{l+1}\right),\\&
     \mathbf{Q}_{l+1}=\mathbf{Q}_l+\alpha_l\left(\mathbf{G}_{l+1}-\mathbf{V}_{l+1}\right), \\&
     \alpha_{l+1}=\mu\alpha_l, \hspace{1em}\mu>1.
   \label{eq: Lag multiplayer and penalty paremeter}
    \end{split}
\end{equation}

\subsection{Identifiability of Co-hub Decomposition}
\label{sec: Identifiability Proof}
\begin{proof}
Let us define $\Delta \mathbf{S}^{k}= \mathbf{S}_{(1)}^{k}-\mathbf{S}_{(2)}^{k}$ and  $\Delta \mathbf{H}= \mathbf{H}_{(1)}-\mathbf{H}_{(2)}$ for $k\in\{1,2,\ldots,K\}$. Since $\mathbf{L}^k=\mathbf{S}_{(1)}^k+\mathbf{H}_{(1)}=\mathbf{S}_{(2)}^k+\mathbf{H}_{(2)},$ 
\begin{align}
\label{iden_eq-1}
   \Delta \mathbf{S}^{k} + \Delta \mathbf{H} = \mathbf{0},\quad \forall k=1,2,\ldots,K . 
\end{align}
If $\Delta \mathbf{H}_{u v} \neq 0$, then either $\left[\mathbf{H}_{(1)}\right]_{ u v}$ or $\left[\mathbf{H}_{(2)}\right]_{ u v}$ must be nonzero. Therefore, for any $(u,v) \in \operatorname{supp}\left(\Delta \mathbf{H}\right)$,
$$
(u,v) \in \operatorname{supp}\left(\mathbf{H}_{(1)}\right)\ \cup\ \operatorname{supp}\left(\mathbf{H}_{(2)}\right).
$$
Equivalently,  
$$
\operatorname{supp}\left(\Delta \mathbf{H}\right) \subseteq \operatorname{supp}\left(\mathbf{H}_{(1)}\right)\ \cup\ \operatorname{supp}\left(\mathbf{H}_{(2)}\right)
$$
From the assumption, we know that $\operatorname{supp}\left(\mathbf{H}_{(i)}\right) \subseteq \mathbb{H}$ for each $i=1,2$ and this will imply $\operatorname{supp}\left(\Delta \mathbf{H}\right) \subseteq \mathbb{H}$. Hence, for any $(u,v) \notin \mathbb{H}$, $ \Delta \mathbf{H}_{uv} =0$. Then from equation (\ref{iden_eq-1}), we can conclude that
$$
\Delta \mathbf{S}^{k}_{uv}=0 \quad \Longrightarrow \quad \left[\mathbf{S}^{k}_{(1)}\right]_{uv}= \left[\mathbf{S}^{k}_{(2)}\right]_{uv},\quad  (u,v) \notin \mathbb{H}.
$$
Now for $(u,v)\in \mathbb{H}$, we have $\Delta \mathbf{S}_{u v}^k=-\Delta \mathbf{H}_{u v}$ from (\ref{iden_eq-1}) for $k\in\{1,2,\ldots,K\}$. Now let us define 
$$
\mathbf{A}_{u v}:=\Delta \mathbf{S}_{u v}^1=\Delta \mathbf{S}_{u v}^2=\ldots = \Delta \mathbf{S}_{u v}^K, \quad(u, v) \in \mathbb{H}.
$$
Then for each $k\in\{1,2,\ldots,K\}$,
\begin{align*}
    \Delta \mathbf{S}^k=\mathbf{A}, \quad \Delta \mathbf{H}=-\mathbf{A}.
\end{align*}
Thus, we can conclude that
$$
\mathbf{S}_{(1)}^k=\mathbf{S}_{(2)}^k+\mathbf{A}, \quad \mathbf{H}_{(1)}=\mathbf{H}_{(2)}-\mathbf{A}.
$$
\end{proof}

\subsection{Proof of the Convergence of the Algorithm}
\label{convergence_proof}
In order to write the optimization problem as a four-block ADMM algorithm, we first define the following matrices:
\begin{itemize}
    \item $\mathbf{R}_K=\mathbf{1}_K \otimes \mathbf{I}_{n^2}$ where $\otimes$ is the Kronecker product. 
    \item$\mathbf{J}_n =\sum_{i=1}^n\left(\mathbf{e}_i \otimes \mathbf{e}_i\right)\left(\mathbf{e}_i \otimes \mathbf{e}_i\right)^{\top} \in \mathbb{R}^{n^2 \times n^2}$, where $\mathbf{e}_i$ is the standard basis vector with the $i$-th element equal to 1 and all other entries equal to zeros. $\widetilde{\mathbf{J}}= \operatorname{bldiag}\left(\mathbf{J}_n,\ldots,\mathbf{J}_n \right).$    
    \item $\mathbf{C}_n =\sum_{i=1}^n \sum_{j=1}^n\left(\mathbf{e}_i \otimes \mathbf{e}_j\right)\left(\mathbf{e}_j \otimes \mathbf{e}_i\right)^{\top}\in \mathbb{R}^{n^2 \times n^2}$.
     \item $\widetilde{\mathbf{P}}:=\operatorname{bldiag}(\underbrace{\mathbf{P}\otimes \mathbf{P}, \ldots, \mathbf{P}\otimes \mathbf{P}}_{K \text { times }}) \in \mathbb{R}^{K n^2 \times K(n-1)^2}$.
    \end{itemize}
We write the optimization problem as a four-block ADMM algorithm as follows,
\begin{equation}
\label{optimization_main_eq}
\min _{\mathbf{x}_1, \mathbf{x}_2, \mathbf{x}_3, \mathbf{x}_4}\  \sum_{i=1}^{4}f_{i}(\mathbf{x}_{i})\quad \text { s.t. } \quad \sum_{i=1}^{4}\mathbf{A}_i \mathbf{x}_i= \mathbf{0},\end{equation}
where the variables  $\mathbf{x}_1, \mathbf{x}_2, \mathbf{x}_3, \mathbf{x}_4$ are defined as  $\mathbf{x}_1=\left[\left\{\operatorname{vec}\left(\mathbf{E}^k\right)\right\}_{k=1}^{K} ,\left\{\operatorname{vec}\left(\mathbf{\Xi}^k\right)\right\}_{k=1}^{K}\operatorname{vec}(\mathbf{V})  \right],\\ \mathbf{x}_2=\left[\left\{\operatorname{vec}\left(\mathbf{Z}^k\right)\right\}_{k=1}^{K} ,\left\{\operatorname{vec}\left(\mathbf{\Gamma}^k\right)\right\}_{k=1}^{K}\right], \mathbf{x}_3= \left[\left\{\operatorname{vec}\left(\mathbf{C}^k\right)\right\}_{k=1}^{K},\left\{\operatorname{vec}\left(\mathbf{\Psi}^k\right)\right\}_{k=1}^{K} \right]$, and $\mathbf{x}_4=\left[\operatorname{vec}(\mathbf{W}),\operatorname{vec}(\mathbf{G})  \right]$ and the corresponding coefficient matrices are as follows, 
$$
\begin{gathered}
\mathbf{A}_1= 
\left[\begin{array}{ccc}
-\widetilde{\mathbf{P}} & \mathbf{0} & \mathbf{0} \\
\mathbf{0} & \mathbf{0} & \mathbf{0} \\
\mathbf{0} & -\widetilde{\mathbf{P}} & \mathbf{0} \\
\mathbf{0} & \mathbf{0} & \mathbf{I}_{n^2} \\
\mathbf{0} & \mathbf{0} & -\mathbf{R}_K \\
\mathbf{0} & \mathbf{0} & -\mathbf{I}_{n^2} \\
-\widetilde{\mathbf{P}} & \mathbf{0} & \mathbf{0}
\end{array}\right]
 ,\quad 
 \mathbf{A}_2=\left[\begin{array}{cc}
\mathbf{0} & \mathbf{0} \\
\mathbf{I}_{K n^2} & \mathbf{0} \\
\mathbf{0} & \mathbf{I}_{K n^2} \\
\mathbf{0} & \mathbf{0} \\
\mathbf{0} & -\mathbf{I}_{K n^2} \\
\mathbf{0} & \mathbf{0} \\
\mathbf{I}_{K n^2} & \mathbf{0}
\end{array}\right] \\
\mathbf{A}_3=\left[\begin{array}{cc}
\mathbf{I}_{K n^2} & \mathbf{0} \\
-\tilde{\mathbf{J}} & \mathbf{0} \\
\mathbf{0} & \mathbf{0} \\
\mathbf{0} & \mathbf{0} \\
\mathbf{I}_{K n^2} & \mathbf{0} \\
\mathbf{0} & \mathbf{0} \\
\mathbf{0} & \mathbf{I}_{K n^2}
\end{array}\right],\quad \mathbf{A}_4= \left[\begin{array}{cc}
\mathbf{0} & \mathbf{0} \\
\mathbf{0} & \mathbf{0} \\
\mathbf{0} & \mathbf{0} \\
-\mathbf{C}_n & \mathbf{0} \\
-\mathbf{R}_K & \mathbf{0} \\
\mathbf{0} & \mathbf{I}_{n^2} \\
\mathbf{0} & \mathbf{0}
\end{array}\right] .
\end{gathered}
$$
The functions $f_1,f_2,f_3,f_4$ are defined as
\begin{align*}
f_1\left(\left\{\mathbf{E}^{k},\mathbf{\Xi}^{k}\right\}, \mathbf{V}\right)=\sum_{k=1}^{K}&\operatorname{tr}\left(\mathbf{B}^{k} \mathbf{E}^{k}\right)+\gamma_4\left\|\mathbf{P} \mathbf{\Xi}^{k} \mathbf{P}^{\top}\right\|_F^2, \quad 
f_2\left(\left\{\mathbf{Z}^{k}, \mathbf{\Gamma}^{k}\right\}\right)=-\sum_{k=1}^{K}\gamma_2 \operatorname{tr}\left(\operatorname{log}\left( \mathbf{Z}^{k}\right)\right), \\
& f_3\left(\left\{\mathbf{C}^{k},\mathbf{\Psi}^{k}\right\}\right)= \sum_{k=1}^K \gamma_1\left\|\boldsymbol{\Psi}^k\right\|_F^2,\quad f_4(\mathbf{W},\mathbf{G})= \gamma_3\|\mathbf{G}\|_{2,1}.
\end{align*}
Now, we can rewrite an equivalent optimization problem of \eqref{optimization_main_eq} as follows,
\begin{equation}
\min _{\mathbf{y}_1, \mathbf{y}_2, \mathbf{y}_3, \mathbf{y}_4}\  \sum_{i=1}^{4}g_{i}(\mathbf{y}_{i})\quad \text { s.t. } \quad \sum_{i=1}^{4}\mathbf{B}_i \mathbf{y}_i= \mathbf{0},\end{equation}
where $\mathbf{y}_1=\mathbf{x}_4$, $\mathbf{y}_2=\mathbf{x}_2$, $\mathbf{y}_3=\mathbf{x}_3$ and $\mathbf{y}_4=\mathbf{x}_1$; $g_1=f_4$, $g_2=f_2$, $g_3=f_3$ and $g_4=f_1$ and $\mathbf{B}_1=\mathbf{A}_4$, $\mathbf{B}_2=\mathbf{A}_2$, $\mathbf{B}_3=\mathbf{A}_3$ and $\mathbf{B}_4=\mathbf{A}_1$. 
The functions $g_1,g_2,g_3,g_4$ are proper, closed, and convex.  The linear equalities for the constraints also provide us with feasibility. Hence, a saddle point $\left(\mathbf{y}_1^*, \mathbf{y}_2^*, \mathbf{y}_3^*, \mathbf{y}_4^*, \boldsymbol{\lambda}^*\right)$ exists. Moreover, the functions $g_2,g_3,g_4$ are strongly convex with the strong convexity parameters $\sigma_2,\sigma_3,\sigma_4$ such that $$
\sigma_2 \geq  \gamma_2 / M^2, \quad \sigma_3 \geq 2 \gamma_1, \quad \sigma_4 \geq 2\gamma_4 .
$$
Therefore, the assumptions $(2.1)$ and $(2.2)$ from \cite{lin2015sublinearconvergenceratemultiblock}
are satisfied. Thus, using theorem $(3.3)$ from \cite{lin2015sublinearconvergenceratemultiblock}, the following are the upper bounds on $\alpha$,$$
\alpha \leq \min_{i=2,3} \frac{2 \sigma_i}{i(2 N-i) \lambda_{\max }\left(\mathbf{B}_i^{\top} \mathbf{B}_i\right)},\ \text{and}
$$
$$
\alpha \leq \frac{2 \sigma_4}{(N-2)(N+1) \lambda_{\max }\left(\mathbf{B}_4^{\top} \mathbf{B}_4\right)},
$$
where $\lambda_{\max}\left(\mathbf{B}\right)$ denotes the largest eigenvalue of the matrix $\mathbf{B}$ and $N$ is the total number of blocks. The largest eigenvalues for $\mathbf{B}_i^{\top} \mathbf{B}_i$ with $i=2,3,4$ are
$$
\lambda_{\max }\left(\mathbf{B}_2^{\top} \mathbf{B}_2\right)=2, \lambda_{\max }\left(\mathbf{B}_3^{\top} \mathbf{B}_3\right)=3, \lambda_{\max }\left(\mathbf{B}_4^{\top} \mathbf{B}_4\right)= K+2. 
$$
Plugging in $N=4$, we have the following upper bound on $\alpha$,
\begin{equation}
\label{alpha-eq}
\alpha \leq \min \left\{\frac{\sigma_2}{6}, \frac{\sigma_3}{15}, \frac{\sigma_4}{5(K+2)}\right\} .
\end{equation}
We then have the following bound,
 \begin{align}
 &\sum_{i=1}^4\left(g_i\left(\bar{\mathbf{y}}_i^t\right)-g_i\left(\mathbf{y}_i^*\right)\right)+\rho\left\|\mathbf{B}_1 \bar{\mathbf{y}}_1^t+\mathbf{B}_2 \bar{\mathbf{y}}_2^t+\mathbf{B}_3 \bar{\mathbf{y}}_3^t+\mathbf{B}_4 \bar{\mathbf{y}}_4^t\right\|_{F}\\
 & \leq \frac{\alpha}{2(t+1)} \sum_{i=1}^3\left\|\sum_{m=i+1}^4 \mathbf{B}_m\left(\mathbf{y}_m^0-\mathbf{y}_m^*\right)\right\|_{F}^2+\frac{\rho^2+\left\|\boldsymbol{\lambda}^0\right\|_{2}^2}{\alpha(t+1)},\end{align}
where  $\mathbf{y}_m^0$ and  $\boldsymbol{\lambda}^0$ are the initial values of $\mathbf{y}_m$ and $\boldsymbol{\lambda}$ respectively. For $\rho= \left\|\boldsymbol{\lambda}^*\right\|_2+1$ and $\alpha$ satisfying $(\ref{alpha-eq})$, we finally have \begin{align*}
        & \,\,\,\,\,\,\,\,\ \left|g\left(\overline{\mathbf{y}}^t\right)-g\left(\mathbf{y}^*\right)\right|=\mathcal{O}(1 / t),\\
          &\left\|\mathbf{B}_1 \overline{\mathbf{y}}_1^t+\mathbf{B}_2 \overline{\mathbf{y}}_2^t+\mathbf{B}_3 \overline{\mathbf{y}}_3^t+\mathbf{B}_4 \overline{\mathbf{y}}_4^t\right\|_{F}=\mathcal{O}(1 / t). 
\end{align*}
Since this optimization problem is equivalent to the main optimization problem in (\ref{optimization_main_eq}), we can conclude that 
\begin{align*}
        & \,\,\,\,\,\,\,\,\ \left|f\left(\overline{\mathbf{x}}^t\right)-f\left(\mathbf{x}^*\right)\right|=\mathcal{O}(1 / t),\\
          &\left\|\mathbf{A}_1 \overline{\mathbf{x}}_1^t+\mathbf{A}_2 \overline{\mathbf{x}}_2^t+\mathbf{A}_3 \overline{\mathbf{x}}_3^t+\mathbf{A}_4 \overline{\mathbf{x}}_4^t\right\|_{F}=\mathcal{O}(1 / t). 
\end{align*}

\subsection{Estimation Error Bound Derivation}
\label{proof_estimation error}
For the proof of Theorem \ref{th:estimation error}, we introduce the following definitions for sub-Gaussianity of the graph signals.  \begin{definition}[Sub-Gaussian Random Vectors]
\label{def:sub-gaussian}
    A random vector $\vx \in \setR^{n}$ is a sub-Gaussian vector if each of its one-dimensional linear projections exhibits sub-Gaussian behavior. More precisely, $\vx$ is considered to be sub-Gaussian if there exists a constant $K > 0$ such that, for any unit vector \(\mathbf{u} \in \mathbb{R}^n\):
    \begin{align}
        \setE \left[ \exp(t \vu^\top (\vx - \setE[\vx])) \right] \leq \exp(K^2 t^2/2)\ \forall t \in \setR.
    \end{align}

    \noindent Moreover, sub-Gaussian norm of $\vx$ is defined as: 
    \begin{align}
        \norm{\vx}_{\psi_2} = \sup_{\vu \in \setS^{n-1}} \norm{\vu^\top \vx}_{\psi_2},
    \end{align}

    \noindent where \(\mathbb{S}^{n-1}\) is the unit sphere in \(\mathbb{R}^n\) and $\norm{\vu^\top \vx}_{\psi_2}$ is the sub-Gaussian norm of random variable $\vu^\top \vx$, defined in \cite{vershynin2010introduction} as follows,
       $$
\|\vu^\top \vx\|_{\psi_2}=\sup _{p \geq 1} p^{-1 / 2}\left(\mathbb{E}|\vu^\top \vx^p|\right)^{1 / p}.
$$\end{definition}
\noindent We also need to introduce the following lemma whose proof is given in \cite{vershynin2010introduction}. 
\begin{lemma}
\label{subG-lem}
    Consider a matrix $\mX \in \setR^{n \times m}$, whose columns are independent sub-Gaussian random vectors with sample covariance  matrix $\widehat{\mathbf{\Sigma}}$. Then, for every $t \geq 0$, the following inequality holds with probability at least $1 - 2\exp(-ct^2)$:
    \begin{align*}
        \norm{\widehat{\mathbf{\Sigma}} - \mathbf{\Sigma}}_2 \leq \max(\delta, \delta^2) \quad \text{where} \quad \delta = C\sqrt{\frac{n}{m}} + \frac{t}{\sqrt{m}},
    \end{align*}
    where $C = C_K$, $c = c_K > 0$ depend only on the sub-Gaussian norm $K = \max_i \| \mX_{\cdot i} \|_{\psi_2}$ of the columns. 
\end{lemma}\label{sec:Estimation Error}
\begin{proof}[Proof of the Theorem \ref{th:estimation error}]
As the estimated $\widehat{\bold{L}}_{\boldsymbol{\gamma}}$ be the minimizer of the rescaled optimization problem in (\ref{rescale_optimization}), then 
\begin{align}
    & \sum^K_{k=1} \left[\frac{1}{d} \operatorname{tr}\left( {\bold{X}^{k}}^\top\widehat{\bold{L}}^{k}_{\boldsymbol{\gamma}} {\bold{X}^{k}}\right) + \gamma_{1d} \left\| \widehat{\bold{L}}^{k}_{\boldsymbol{\gamma} }-\bold{I}\odot\widehat{\bold{L}}^{k}_{\boldsymbol{\gamma}}\right\|_{F}^2  - \gamma_{2d} \operatorname{tr}\left(\log\left(\bold{I}\odot \widehat{\bold{L}}^{k}_{\boldsymbol{\gamma}}\right)\right) + \gamma_{4d} \left\|\widehat{\bold{S}}_{\boldsymbol{\gamma}}^{k}\right\|_{F}^{2} \right]+\gamma_{3d} \left\| \widehat{\bold{V}}_{\boldsymbol{\gamma}}\right\|_{2,1}\notag\\ 
    & \leq \sum^K_{k=1} \left[\frac{1}{d} \operatorname{tr}\left( {\bold{X}^{k}}^\top{\bold{L}^{k}}^* {\bold{X}^{k}}\right) + \gamma_{1d} \left\| {\bold{L}^{k}}^* -\bold{I}\odot{\bold{L}^{k
    }}^*\right\|_{F}^2  - \gamma_{2d} \operatorname{tr} \left(\log\left(\bold{I}\odot{\bold{L}^{k}}^*\right)\right)+ \gamma_{4d} \left\|{\bold{S}^{k}}^*\right\|_{F}^{2}\right]+\gamma_{3d} \left\| {\bold{V}}^*\right\|_{2,1}\ . 
\end{align}
Defining $\widehat{\bold{\Sigma}}^{k}= \frac{1}{d} \bold{X}^{k} {\bold{X}^{k}}^{\top}$, and we can rewrite the above inequality as,
\begin{align}
   & \sum^K_{k=1} \operatorname{tr}\left( \left(\widehat{\bold{L}}_{\boldsymbol{\gamma}}^{k}-{\bold{L}^{k}}^*\right) \widehat{\bold{\Sigma}}^{k}\right) + \gamma_{1d} \left\| \widehat{\bold{L}}^{k}_{\boldsymbol{\gamma} }-\bold{I}\odot\widehat{\bold{L}}^{k}_{\boldsymbol{\gamma}}\right\|_{F}^2 -  \gamma_{2d} \operatorname{tr}\left(\log\left(\bold{I}\odot \widehat{\bold{L}}^{k}_{\boldsymbol{\gamma}}\right)\right)- \gamma_{1d} \left\| {\bold{L}^{k}}^*-\bold{I}\odot {\bold{L}^{k}}^*\right\|_{F}^2 \notag \\
    & + \gamma_{2d} \operatorname{tr} \left(\log\left(\bold{I}\odot {\bold{L}^{k}}^*\right)\right)\leq  \gamma_{3d} \left\{\left\| {\bold{V}}^*\right\|_{2,1}- \left\| \widehat{\bold{V}}_{\boldsymbol{\gamma}}\right\|_{2,1}\right\} + \gamma_{4d} \sum^K_{k=1} \left\{\left\|{\bold{S}^{k}}^*\right\|_{F}^{2}- \left\|\widehat{\bold{S}}_{\boldsymbol{\gamma}}^{k}\right\|_{F}^{2}\right\}  \notag .
\end{align}
If ${\mathbf{\Sigma}^{k}}^*$ is the true covariance matrix corresponding to view $k$ then we  further have,
\begin{align}
\label{main-ineq}
    & \sum^K_{k=1}\left[ \operatorname{tr}\left( \left(\widehat{\bold{L}}_{\boldsymbol{\gamma}}^{k}-{\bold{L}^{k}}^{*}\right) {\bold{\Sigma}^{k}}^*\right) + \gamma_{1d} \left\| \widehat{\bold{L}}^{k}_{\boldsymbol{\gamma} }-\bold{I}\odot\widehat{\bold{L}}^{k}_{\boldsymbol{\gamma}}\right\|_{F}^2
     -  \gamma_{2d} \operatorname{tr}\left(\log\left(\bold{I}\odot \widehat{\bold{L}}^{k}_{\boldsymbol{\gamma}}\right)\right)- \gamma_{1d} \left\| {\bold{L}^{k}}^* -\bold{I}\odot {\bold{L}^{k
    }}^* \right\|_{F}^2 +\gamma_{2d} \operatorname{tr} \left( \log\left(\bold{I} \odot {\bold{L}^{k}}^{*}\right) \right)\right]\notag\\
    & \leq \gamma_{3d} \left( \left\| {\bold{V}}^* \right\|_{2,1} - \left\| \widehat{\bold{V}}_{\boldsymbol{\gamma}} \right\|_{2,1} \right)+\gamma_{4d} \sum^K_{k=1} \left\{\left\|{\bold{S}^{k}}^*\right\|_{F}^{2}- \left\|\widehat{\bold{S}}_{\boldsymbol{\gamma}}^{k}\right\|_{F}^{2}\right\}+  \sum_{k=1}^{K} \operatorname{tr} \left( \left({\bold{L}^{k}}^{*} - \widehat{\bold{L}}_{\boldsymbol{\gamma}}^{k}\right) \left(\widehat{\bold{\Sigma}}^{k} - {\bold{\Sigma}^{k}}^{*}\right) \right).
\end{align}
For a fixed view $k$, let us define
\begin{align*}
 G_k(\mathbf{L}^{k}): & =\operatorname{tr}\left(\mathbf{L}^{k} \boldsymbol{\Sigma}^k\right)+\gamma_{1 d}\left\|\mathbf{L}^{k}-\mathbf{I} \odot \mathbf{L}^{k}\right\|_{F}^2 -\gamma_{2 d} \operatorname{tr}(\log (\mathbf{I} \odot \mathbf{L}^{k})).   
\end{align*}
A first-order Taylor expansion at $\mathbf{L}^k$ gives the exact decomposition
$$
G_k\left(\mathbf{L}^k+\mathbf{\Delta}^k\right)-G_k\left(\mathbf{L}^k\right)=\left\langle\nabla G_k\left(\mathbf{L}^k\right), \mathbf{\Delta}^k\right\rangle+\mathcal{R}_k\left(\mathbf{L}^k ; \mathbf{\Delta}^k\right),
$$
where the remainder along the segment $\mathbf{L}^k(\tau):=\mathbf{L}^k+\tau \mathbf{\Delta}^k, \tau \in[0,1]$ $$\mathcal{R}_k\left(\mathbf{L}^k ; \mathbf{\Delta}^k\right)=\int_0^1(1-\tau)\left\langle\mathbf{\Delta}^k, \boldsymbol{\nabla}^2 G_k\left(\mathbf{L}^k(\tau)\right) \mathbf{\Delta}^k\right\rangle d \tau.$$ Now, the gradient calculated for each entry will be 
$$
\left[\nabla G_k(\mathbf{L}^{k})\right]_{i j}= \begin{cases}\Sigma_{i j}^k+2 \gamma_{1 d} L^{k}_{i j}, & i \neq j, \\ \Sigma_{i i}^k-\frac{\gamma_{2 d}}{L_{i i}^k}, & i=j.\end{cases}
$$
The quadratic expression associated with the second derivative is given as:
$$\left\langle\mathbf{\Delta}^{k}, \mathbf{\nabla}^2 G_k(\mathbf{L}^{k}) \mathbf{\Delta}^{k}\right\rangle=2 \gamma_{1 d} \sum_{i \neq j} {\Delta_{i j}^k}^2+\gamma_{2 d} \sum_{i=1}^n \frac{{\Delta_{i i}^k}^2}{{L_{i i}^k}^2}.$$
Consequently, along the path $\mathbf{L}^k(\tau)$ we have
$$
\left\langle\mathbf{\Delta}^k, \mathbf{\nabla}^2 G_k\left(\mathbf{L}^k(\tau)\right) \mathbf{\Delta}^k\right\rangle \geq \mu_k(\tau)\left\|\mathbf{\Delta}^k\right\|_F^2,  
$$
$\mu_k(\tau):=\min \left\{2 \gamma_{1 d}, \min _{1 \leq i \leq n} \frac{\gamma_{2 d}}{\left(L_{i i}^k+\tau \Delta_{i i}^k\right)^2}\right\}$.  Using assumption $(A3)$, we have     $$
\mu_k(\tau) \geq \mu_k:=\min \left\{2 \gamma_{1 d}, \frac{\gamma_{2 d}}{M_k^2}\right\}.
$$ 
Since $\int_0^1(1-\tau) d \tau=\frac{1}{2}$, the lower bound for each view is
$$
G_k\left(\mathbf{L}^k+\mathbf{\Delta}^k\right)-G_k\left(\mathbf{L}^k\right) \geq\left\langle \boldsymbol{\nabla} G_k\left(\mathbf{L}^k\right), \mathbf{\Delta}^k\right\rangle+\frac{\mu_k}{2}\left\|\mathbf{\Delta}^k\right\|_F^2.
$$
Substituting ${\mathbf{L}^{k}}^{*}$ for $\mathbf{L}^{k}$ and  $\mathbf{\Delta}^{k}= \widehat{\mathbf{L}}_{\boldsymbol{\gamma}}^k-{\mathbf{L}^{k}}^{*}$ and summing over $k=1, \ldots, K$ and denoting $\mu:=\min _k \mu_k$, we obtain
\begin{align}
\label{G_k-equation}
& \sum_{k=1}^K\left(G_k\left(\widehat{\mathbf{L}}_{\boldsymbol{\gamma}}^k\right)-G_k\left({\mathbf{L}^{k}}^{*}\right)\right)  \geq \sum_{k=1}^K\left\langle\boldsymbol{\nabla} G_k\left({\mathbf{L}^{k}}^{*}\right), \widehat{\mathbf{L}}_{\boldsymbol{\gamma}}^k-{\mathbf{L}^k}^{*}\right\rangle +\frac{\mu}{2} \sum_{k=1}^K\left\|\widehat{\mathbf{L}}_{\boldsymbol{\gamma}}^k-{\mathbf{L}^k}^{*}\right\|_F^2 .
\end{align}
Applying Cauchy-Schwarz inequality to each term gives the magnitude bound
\begin{align*}
   &\left|\left\langle\boldsymbol{\nabla} G_k\left({\mathbf{L}^k}^{*}\right), \mathbf{\Delta}^k\right\rangle\right| \leq\left(\left\|{\boldsymbol{\Sigma}^k}^*\right\|_F+\gamma_{2 d}\left\|\left(\mathbf{I} \odot {\mathbf{L}^k}^{*}\right)^{-1}\right\|_F\right)\left\|\mathbf{\Delta}^k\right\|_F +2 \gamma_{1 d}\left\|{\mathbf{L}^k}^{*}-\mathbf{I} \odot {\mathbf{L}^k}^{*}\right\|_F\left\|\mathbf{\Delta}^k\right\|_F .
\end{align*}
Summing over $k=1,2,\ldots, K$, we get
\begin{align*}
    \sum_{k=1}^K\left\langle\boldsymbol{\nabla} G_k\left({\mathbf{L}^k}^{*}\right), \mathbf{\Delta}^k\right\rangle & \geq-\sum_{k=1}^K C_k\left\|\mathbf{\Delta}^k\right\|_F \geq-\left(\max _{1 \leq k \leq K} C_k\right) \sum_{k=1}^K\left\|\mathbf{\Delta}^k\right\|_F, 
\end{align*}
where the constant $
C_k:=\left\|{\mathbf{\Sigma}^k}^{*}\right\|_F+2 \gamma_{1 d}\left\|{\mathbf{L}^k}^{*}-\mathbf{I} \odot {\mathbf{L}^k}^{*}\right\|_F+\gamma_{2 d}\left\|\left(\mathbf{I} \odot {\mathbf{L}^k}^{*}\right)^{-1}\right\|_F
$. Again applying Cauchy-Schwarz inequality and using the fact $\|\mathbf{\Delta}\|_F^2:=\sum_{k=1}^K\left\|\mathbf{\Delta}^k\right\|_F^2$, we have \begin{align}
  \label{G_k_Equation 2}\sum_{k=1}^K\left\langle\boldsymbol{\nabla} G_k\left({\mathbf{L}^k}^{*}\right), \mathbf{\Delta}^k\right\rangle \geq -\left(\max _{1 \leq k \leq K} C_k\right) \sqrt{K}\|\mathbf{\Delta}\|_F.  
\end{align}
Combining equations (\ref{G_k-equation}) and (\ref{G_k_Equation 2}) and denoting $C^{\prime}= \max _{1 \leq k \leq K} C_k$
\begin{align*}
\sum_{k=1}^K\left(G_k\left(\widehat{\mathbf{L}}_{\boldsymbol{\gamma}}^k\right)-G_k\left({\mathbf{L}^k}^{*}\right)\right) & \geq -C^{\prime} \sqrt{K}\left\|\widehat{\mathbf{L}}_{\boldsymbol{\gamma}}^k-{\mathbf{L}^k}^{*}\right\|_F +\frac{\mu}{2} \sum_{k=1}^K\left\|\widehat{\mathbf{L}}_{\boldsymbol{\gamma}}^k-{\mathbf{L}^k}^{*}\right\|_F^2 .
\end{align*}For the penalizing term, let us write the index set as a disjoint union
$\{1, \ldots, n\}=H \cup H^c$
where $H=\left\{j:\left\|\mathbf{V}_{\cdot j}^*\right\|_2 \neq 0\right\}$ with cardinality $|H|=h \ll n$. For  notational simplicity, we use $\widehat{\mathbf{V}}_{\cdot j}$ to denote the $j$-th column of $\widehat{\mathbf{V}}_{\boldsymbol{\gamma}}$. Then  
$$
\begin{aligned}
\left\|\mathbf{V}^*\right\|_{2,1}-\|\widehat{\mathbf{V}}_{\boldsymbol{\gamma}}\|_{2,1}
& = \sum_{j \in H}\left\|\mathbf{V}_{\cdot j}^*\right\|_2-\left(\sum_{j \in H}\left\|\widehat{\mathbf{V}}_{\cdot j}\right\|_2+\sum_{j \in H^c}\left\|\widehat{\mathbf{V}}_{\cdot j}\right\|_2\right) \\
& =\sum_{j \in H}\left(\left\|\mathbf{V}_{\cdot j}^*\right\|_2-\left\|\widehat{\mathbf{V}}_{\cdot j}\right\|_2\right)-\underbrace{\sum_{j \in H^c}\left\|\widehat{\mathbf{V}}_{\cdot j}\right\|_2}_{\geq 0} .\end{aligned}
$$
Since the second term is non-negative, we can drop it to get $$
\left\|\mathbf{V}^*\right\|_{2,1}-\|\widehat{\mathbf{V}}\|_{2,1} \leq \sum_{j \in H}\left(\left\|\mathbf{V}_{\cdot j}^*\right\|_2-\left\|\widehat{\mathbf{V}}_{\cdot j}\right\|_2\right) .
$$ Using the triangle inequality, we have \begin{align*}
&  \,\,\,\,\,\,\,\,\,\,\,\,\,\,\ \left\|\mathbf{V}_{\cdot j}^*\right\|_2-\left\|\widehat{\mathbf{V}}_{\cdot j}\right\|_2 \leq\left\|\mathbf{V}_{\cdot j}^*-\widehat{\mathbf{V}}_{\cdot j}\right\|_2  \\
& \Longrightarrow\,\,\,\, \left\|\mathbf{V}^*\right\|_{2,1}-\|\widehat{\mathbf{V}}_{\boldsymbol{\gamma}}\|_{2,1} \leq \sum_{j \in H}\left\|\widetilde{\mathbf{V}}_{\cdot j}\right\|_2,
\end{align*}
 where $\widetilde{\mathbf{V}}= \widehat{\mathbf{V}}_{\boldsymbol{\gamma}}-\mathbf{V}^*$. The well-known Cauchy–Schwarz inequality provides us the following bound,\begin{align*}
     \sum_{j \in H}\left\|\widetilde{ \mathbf{V}}_{\cdot j}\right\|_2 & \leq \sqrt{h}\left(\sum_{j \in H}\left\|\widetilde{ \mathbf{V}}_{\cdot j}\right\|_2^2\right)^{1 / 2} \leq \sqrt{h}\|\widetilde{ \mathbf{V}}\|_F \leq \sqrt{h}\|\mathbf{\Delta}\|_F.
     \end{align*}
For the term dependent on the data, let us define the events
\[
E_k \;:=\;\left\{
\left\|\widehat{\mathbf{\Sigma}}^{k}-{\mathbf{\Sigma}^{k}}^{*}\right\|_F
\le \frac{n}{\sqrt d}\,C^{k}
\right\},\qquad k=1,\dots,K.
\] Using Lemma \ref{subG-lem}  with $t=a \sqrt{n}$ where $a \geq 0$ and $d \geq\left(C_K+a\right)^2 n$, we have  \(\mathbb{P}(E_k^{\mathrm c}) \le 2 e^{-c^{k} a^2 n}\). Using Boole’s inequality,
\begin{align}
\mathbb{P}\!\left(\bigcup_{k=1}^K E_k^{\mathrm c}\right)
&\le \sum_{k=1}^K \mathbb{P}(E_k^{\mathrm c})
\;\le\; 2K\,e^{-c^{k} a^2 n}.
\end{align} 
Taking complements, we get the bound
\begin{equation}
\mathbb{P}\!\left(\bigcap_{k=1}^K E_k\right) = 1- \mathbb{P}\!\left(\bigcup_{k=1}^K E_k^{\mathrm c}\right)\;\ge\; 1 - 2K\,e^{-c^k a^2 n}.
\label{eq:simul}
\end{equation}
If we choose the constant $a \geq \sqrt{\frac{\log (2 K)}{c^k n}}$, then the above probability becomes a valid one. Using Cauchy–Schwarz
inequality and the well-known result ,
$
\|\mathbf{A}\|_F \leq \sqrt{d}\|\mathbf{A}\|_2,
$ for any matrix $\mathbf{A} \in \mathbb{R}^{d \times d}$, we derive the following bound,
\begin{align}
      & \sum_{k=1}^{K} \operatorname{tr} \left(\left( \bold{L}^{k^{*}}-\widehat{\bold{L}}_{\boldsymbol{\gamma}}^{k}\right)\left(\widehat{\mathbf{\Sigma}}^{k}- \mathbf{\Sigma}^{k^*}\right)\right) \leq \sum_{k=1}^{K} \left\|\widehat{\bold{L}}_{\boldsymbol{\gamma}}^{k}- \bold{L}^{k^{*}}\right\|_{F} \left\|\widehat{\mathbf{\Sigma}}^{k}- \mathbf{\Sigma}^{k^*}\right\|_{F} \leq \sqrt{d} \sum_{k=1}^{K}\left\|\widehat{\bold{L}}_{\boldsymbol{\gamma}}^{k}- \bold{L}^{k^{*}}\right\|_{F} \left\|\widehat{\mathbf{\Sigma}}^{k}- \mathbf{\Sigma}^{k^*}\right\|_{2}\notag \\
      & \leq  \frac{n}{\sqrt{d}}\sum_{k=1}^{K} C^{k} \left\|\widehat{\bold{L}}_{\boldsymbol{\gamma}}^{k}- \bold{L}^{k^{*}}\right\|_{F} \leq \frac{Kn}{\sqrt{d}}\widetilde{C}^{k} \left\|\widehat{\bold{L}}_{\boldsymbol{\gamma}}- \bold{L}^{*}\right\|_{F}\ , 
  \end{align}
  with probability at least $1 - 2K\,e^{-c^k a^2 n}$ and taking $\widetilde{C}^{k}=\operatorname{max}\left\{C^{1},C^{2},\ldots,C^{K}\right\}$. Thus, we have 
$$
\gamma_{4 d} \sum_{k=1}^K\left\{\left\|\mathbf{S}^{k^ *}\right\|_F^2-\left\|\widehat{\mathbf{S}}_{\boldsymbol{\gamma}}^{k}\right\|_F^2\right\} \leq \gamma_{4 d} \sum_{k=1}^K\left\|\mathbf{S}^{k^ *}\right\|_F^2=: \gamma_{4 d} C_{S^*}.
$$
Combining all of the inequalities in  (\ref{main-ineq}) we have,
\begin{align}
& -C^{\prime} \sqrt{K}\left\|\widehat{\mathbf{L}}_{\boldsymbol{\gamma}}^k-\mathbf{L}^*\right\|_F+ \frac{\mu}{2} \sum_{k=1}^K\left\|\widehat{\mathbf{L}}_{\boldsymbol{\gamma}}^k-\mathbf{L}^k\right\|_F^2  \leq \frac{K n}{\sqrt{d}} \widetilde{C}^k\left\|\widehat{\mathbf{L}}_{\boldsymbol{\gamma}}-\mathbf{L}^*\right\|_F+\gamma_{3d}\sqrt{h}\left\|\widehat{\mathbf{L}}_{\boldsymbol{\gamma}}^k-\mathbf{L}^*\right\|_F+ \gamma_{4 d} C_{S^*}.\end{align}
We can write it as a quadratic inequality as $a x^2 \leq b x+c$ with    $a= \frac{\mu}{2}, b=C^{\prime} \sqrt{K}+\gamma_{3 d} \sqrt{h}+\frac{K n}{\sqrt{d}} \widetilde{C},  c=\gamma_{4 d} C_{S^*}$ and $x= \left\|\widehat{\mathbf{L}}_{\boldsymbol{\gamma}}^k-\mathbf{L}^*\right\|_F=\left\|\mathbf{\Delta}\right\|_{F}$.  The positive root of $a x^2-b x-c=0$ gives the upper bound:\begin{align}
    \|\mathbf{\Delta}\|_F \leq \frac{b+\sqrt{b^2+4ac}}{2a}.
    \end{align}
Substituting the values of $a$, $b$ and $c$, we finally get the following inequality,
\begin{align}
\|\mathbf{\Delta}\|_F \leq \frac{2K n}{\mu\sqrt{d}} \widetilde{C}+\frac{2}{\mu}\left(C^{\prime} \sqrt{K}+\gamma_{3 d} \sqrt{h}+\sqrt{\frac{ \mu\gamma_{4 d}C_{S^*}}{2}}\right).
\end{align}
    
\end{proof}

\bib

\end{document}